%% file: paper.tex
\documentclass{article}

\usepackage{microtype}
\usepackage{graphicx}
\usepackage{subcaption}
\usepackage{booktabs}
\usepackage{hyperref}

\usepackage[accepted]{icml2026}
\makeatletter
\renewcommand{\ICML@appearing}{\textit{Accepted to the 2nd Workshop on Compositional Learning at ICML 2026, Seoul, South Korea. Copyright 2026 by the author(s).}}
\makeatother

\usepackage{amsmath}
\usepackage{amssymb}
\usepackage{mathtools}
\usepackage{amsthm}
\usepackage[capitalize,noabbrev]{cleveref}

\theoremstyle{plain}
\newtheorem{theorem}{Theorem}[section]
\newtheorem{proposition}[theorem]{Proposition}
\newtheorem{lemma}[theorem]{Lemma}
\newtheorem{corollary}[theorem]{Corollary}
\theoremstyle{definition}
\newtheorem{definition}[theorem]{Definition}
\newtheorem{assumption}[theorem]{Assumption}
\theoremstyle{remark}
\newtheorem{remark}[theorem]{Remark}

\newcommand{\cH}{\mathcal{H}}
\newcommand{\cC}{\mathcal{C}}
\newcommand{\cF}{\mathcal{F}}
\newcommand{\cG}{\mathcal{G}}
\newcommand{\cX}{\mathcal{X}}
\newcommand{\cY}{\mathcal{Y}}
\newcommand{\RR}{\mathbb{R}}
\newcommand{\EE}{\mathbb{E}}
\newcommand{\PP}{\mathbb{P}}
\newcommand{\norm}[1]{\left\lVert #1 \right\rVert}
\newcommand{\lip}{\mathrm{Lip}}
\newcommand{\Rad}{\mathfrak{R}}
\newcommand{\wh}{\widehat}

\icmltitlerunning{Sample Complexity of Scientific Discovery: PAC Learnability of Compositional Function Trees}

\begin{document}

\twocolumn[
\icmltitle{Sample Complexity of Scientific Discovery: \\
PAC Learnability of Compositional Function Trees}

\begin{icmlauthorlist}
  \icmlauthor{\c{S}uayp Talha Kocabay}{indep,tubitak}
  \icmlauthor{Talha R\"{u}zgar Akku\c{s}}{indep}
  \icmlauthor{Kerem Yal\c{c}{\i}n}{indep}
\end{icmlauthorlist}

\icmlaffiliation{indep}{Independent Researcher}
\icmlaffiliation{tubitak}{T\"{U}B\.{I}TAK Science High School, T\"{u}rkiye}

\icmlcorrespondingauthor{\c{S}uayp Talha Kocabay}{suayptalhakocabay.28@tubitakfenlisesi.k12.tr}
\icmlcorrespondingauthor{Talha R\"{u}zgar Akku\c{s}}{talharuzgarakkus@gmail.com}
\icmlcorrespondingauthor{Kerem Yal\c{c}{\i}n}{keremyalcin.28@tubitakfenlisesi.k12.tr}

\icmlkeywords{Sample Complexity; PAC Learning; Rademacher Complexity; Symbolic Regression; Compositional Models}

\vskip 0.3in
]

\printAffiliationsAndNotice{}

\begin{abstract}
Scientific discovery via symbolic regression is often viewed as statistically and computationally intractable because the hypothesis space of expressions grows combinatorially with depth.
This paper revisits the statistical side through the lens of PAC learning, focusing on compositional function trees built from a finite vocabulary of smooth operators (e.g., $\{+,\times,\sin,\exp\}$ and affine maps).
We prove that the relevant generalization quantity---Rademacher complexity, hence the excess risk---does \emph{not} necessarily blow up exponentially with the number of distinct symbolic structures, but is controlled by (i) the depth $d$ and (ii) the Lipschitz constants of the base operators along the composed computation graph.
Concretely, under mild Lipschitz conditions on operators and bounded affine leaves, a finite-union bound over a vocabulary of size $K{=}|\cH_{\mathrm{base}}|$ together with Maurer-type vector contraction yields $\Rad_n(\cH_{\mathrm{comp}}^{d}) \le (Kb\sqrt{2}\,L)^{d-1}\Rad_n(\cH_{\mathrm{comp}}^{1})$ with arity bound $b$; corresponding high-probability risk bounds scale as $\mathcal{O}(L^{d}/\sqrt{n})$ when $K,b{=}O(1)$ and $\Rad_n(\cH_{\mathrm{comp}}^{1}){=}O(n^{-1/2})$.
We complement the theory with a modular codebase that trains differentiable operator trees (not MLPs) on synthetic ``physics-like'' targets of controlled depth and shows that the empirical generalization gap correlates positively with the predicted complexity term $(\wh L^{d})/\sqrt{n}$.
\end{abstract}

\section{Introduction}
Scientific machine learning frequently requires \emph{interpretable} functional relationships from small datasets, not only accurate predictions.
While deep neural networks can approximate complex functions, they typically do not yield concise symbolic laws, raising concerns in the sciences where interpretability, extrapolation, and falsifiability are central.
Symbolic regression (SR) addresses this by searching over expressions built from a vocabulary of primitives, and has led to influential systems and benchmarks \cite{koza1992genetic,schmidt2009distilling,schmidt2009implicit,udrescu2020aifeynman,cranmer2023pysr,la2021srbench}.

However, SR is often described as ``unlearnable'' because the number of candidate expressions grows exponentially with depth, and finding the best expression is NP-hard in many formulations.
This view conflates \emph{computational} hardness with \emph{statistical} learnability.
PAC learning theory \cite{valiant1984theory} separates these issues: a hypothesis class can be PAC learnable (small sample complexity) even if optimizing over it is computationally challenging.

We formalize this separation for a compositional hypothesis space motivated by scientific discovery.
Our core message is that statistical generalization of compositional operator trees is governed by compositional \emph{stability} (Lipschitzness) rather than by the raw count of symbolic structures.
Technically, we control generalization via Rademacher complexity \cite{bartlett2002rademacher,mohri2018foundations}, extending classical contraction arguments \cite{ledoux1991probability} to tree-structured compositions.

Discrete search procedures incur computational costs proportional to the number of structures explored \cite{koza1992genetic,petersen2021equation}, whereas uniform convergence controls how reliably empirical risk ranks hypotheses once attention is restricted to a class $\cH_{\mathrm{comp}}^{d}$ \emph{and its continuous parameters}.
In practice SR couples both: a learner proposes syntax trees and fits coefficients (often nonlinearly); our bounds clarify what statistical scaling one should expect from that nonlinear estimation step when stability holds across primitives.

\paragraph{Contributions.}
\begin{itemize}
  \item We define a compositional function-tree hypothesis space $\cH_{\mathrm{comp}}^{d}$ capturing symbolic expressions of depth $d$ built from a base vocabulary $\cH_{\mathrm{base}}$.
  \item Under mild Lipschitz assumptions on base operators and bounded affine leaves (\cref{ass:affine-bound}), we bound $\Rad_n(\cH_{\mathrm{comp}}^{d})$ by $(Kb\sqrt{2}\,L)^{d-1}\Rad_n(\cH_{\mathrm{comp}}^{1})$ for $K{=}|\cH_{\mathrm{base}}|$ (\cref{eq:rad-depth,lem:finite-union}), yielding PAC-style risk bounds scaling as $\mathcal{O}(L^{d}/\sqrt{n})$ when $K,b{=}O(1)$.
  \item We provide a clean PyTorch implementation of differentiable compositional trees and controlled-depth synthetic ``physics'' datasets, and empirically verify that the generalization gap tracks $(\wh L^{d})/\sqrt{n}$.
\end{itemize}

\paragraph{Theory versus empirical proxies.}
Our guarantees concern uniform convergence over $\cH_{\mathrm{comp}}^{d}$ and invoke population quantities $(d,L)$ (localized when nonlinear operators such as $\exp$ appear).
The experiments substitute $d$ by the fixed structural depth of the fitted tree and $L$ by a finite-batch estimator $\wh L$.
Thus alignment between observed gaps and $(\wh L^{d})/\sqrt{n}$ should be read as evidence that compositional stability predicts qualitative scaling rather than as identification of sharp numerical constants.

\section{Related Work}
\paragraph{Symbolic regression and scientific discovery.}
Classic SR is commonly instantiated via genetic programming \cite{koza1992genetic} or program synthesis-style search, with modern toolchains emphasizing physics heuristics \cite{udrescu2020aifeynman}, large-scale benchmarking \cite{la2021srbench}, and differentiable relaxations \cite{cranmer2020discovering,cranmer2023pysr}.
Recent work has also explored deep reinforcement learning and risk-seeking objectives for expression discovery \cite{petersen2021equation}, as well as broader surveys of SR methodology and applications \cite{makke2024symbolic}.
In parallel, sparse discovery methods such as SINDy \cite{brunton2016sindy} target dynamical systems by searching over a pre-specified library with sparsity penalties.
Our focus differs: we do not propose a new SR search algorithm; instead, we give a statistical learnability analysis of \emph{compositional} operator classes and validate the predicted scaling empirically.

\paragraph{Generalization via complexity and stability.}
Rademacher complexity and related uniform convergence tools are standard in learning theory \cite{vapnik1998statistical,bartlett2002rademacher,shalev2014understanding,mohri2018foundations}.
For deep neural networks, stability and Lipschitz-like quantities (e.g., spectral norms) play central roles in generalization bounds \cite{bartlett2017spectrally}.
Our analysis is closest in spirit to compositional complexity controls (e.g., chain rules for Gaussian/Rademacher processes and vector contraction inequalities) \cite{ledoux1991probability,talagrand2005generic,vandervaart1996empirical,maurer2016chain,maurer2016vectorcontraction}, specialized here to operator trees commonly used in scientific modeling.

\paragraph{Depth versus width.}
Much neural-network theory emphasizes architectural depth jointly with width, norms, or margin quantities \cite{bartlett2017spectrally,neyshabur2017exploring}.
Here depth enters explicitly through repeated composition of a \emph{fixed vocabulary} of interpretable operators and produces multiplicative Lipschitz accumulation along trees rather than implicit coupling through parameters.
This distinction clarifies why symbolic/compositional models admit bounds stated directly in terms of $(d,L)$ without appealing to implicit layers or infinite-width limits.

\paragraph{Refined contraction and localization.}
Recent refinements of Talagrand-type contraction for structured vector-valued maps sharpen dependence on norms and layer geometry beyond the textbook lemma \cite{truong2022rademacher}.
Parallel localization tools---covering-number-based local Rademacher complexities \cite{lei2016local} and offset/local frameworks that yield estimator-aware excess-risk bounds \cite{kanade2024exponential}---provide routes toward sharper or hypothesis-dependent guarantees than the uniform bounds emphasized here.

\paragraph{Discrete syntax graphs versus continuous estimation.}
Much SR methodology emphasizes discrete structure discovery combined with numerical tuning \cite{koza1992genetic,udrescu2020aifeynman,cranmer2023pysr}.
From a learning-theoretic viewpoint it is helpful to separate (i) the outer discrete optimization problem---often akin to program synthesis \cite{petersen2021equation}---from (ii) the inner regression problem induced once an architecture family is fixed.
Our complexity estimates speak primarily to (ii): depth enters compositionally through gradients/Lipschitz maps rather than through enumeration over exponentially many graphs.

\section{Preliminaries}
\label{sec:prelim}
\subsection{Learning setup}
Let $\cX \subset \RR^{p}$ and $\cY \subset \RR$.
Given i.i.d.\ data $S=\{(x_i,y_i)\}_{i=1}^{n}$ drawn from an unknown distribution $\mathcal{D}$ on $\cX\times \cY$, we consider squared loss $\ell(f(x),y) = (f(x)-y)^2$.
For a hypothesis $f\in\cH$, define true risk $R(f)=\EE_{(x,y)\sim\mathcal{D}}[\ell(f(x),y)]$ and empirical risk $\wh R_S(f)=\frac{1}{n}\sum_{i=1}^n \ell(f(x_i),y_i)$.

\subsection{Trees, depth, and branching}
To make the role of \emph{depth} explicit, it is useful to describe a hypothesis $f\in\cH_{\mathrm{comp}}^{d}$ as a rooted operator tree whose internal nodes are operators from $\cH_{\mathrm{base}}$ and whose leaves are either input coordinates or constants.
Let $b$ denote the maximum arity (branching factor) of operators (for our vocabulary, $b\le2$).
We distinguish:
\begin{itemize}
  \item \textbf{Depth} $d$: maximum number of operator applications along any root-to-leaf path.
  \item \textbf{Size} $s$: number of internal nodes (operators) in the tree.
\end{itemize}
While the number of distinct trees of depth $d$ can grow super-exponentially in $d$, our bounds depend on $d$ (and weakly on $b$) through stability, not enumeration.

\subsection{Base operator class and compositional trees}
Let $\cH_{\mathrm{base}}$ be a finite set of (possibly parameterized) operators acting on scalars or tuples of scalars.
In this paper we focus on a small differentiable vocabulary typical of SR and physics-inspired modeling:
\begin{equation}
  \cH_{\mathrm{base}} \subset \Big\{\mathrm{add},\ \mathrm{mult},\ \sin,\ \cos,\ \exp,\ \mathrm{affine}\Big\}.
\end{equation}
We define a \emph{compositional hypothesis space of depth $d$} recursively.
For simplicity, assume a binary tree arity (unary operators can be modeled by taking one child to be a constant function).

\begin{definition}[Compositional hypothesis space]
Let $\cH_{\mathrm{comp}}^{1}$ be the set of base (non-composed) hypotheses. For $d\ge2$, define
\begin{equation}
  \cH_{\mathrm{comp}}^{d}
  =
  \Big\{
    h \circ (g_1,g_2)\ :\ h\in\cH_{\mathrm{base}},\ g_1,g_2\in\cH_{\mathrm{comp}}^{d-1}
  \Big\}.
  \label{eq:compclass}
\end{equation}
\end{definition}
Here $(g_1,g_2)$ denotes the pair-valued map $x\mapsto (g_1(x),g_2(x))$ and $h$ acts on $\RR^2$ (or $\RR$ in the unary case).
The recursion induces a computation tree with depth $d$.

\paragraph{Parameters.}
Leaves labeled $\mathrm{affine}$ introduce slope/intercept parameters; compositions propagate gradients through operators exactly as in standard automatic differentiation.

\begin{assumption}[Bounded affine leaves]
\label{ass:affine-bound}
Assume $\mathcal{X}$ has bounded coordinates: $\norm{x}_\infty\le X_{\max}$ almost surely.
Every affine primitive $x\mapsto \theta^\top x + \theta_0$ used at a leaf satisfies $\norm{\theta}_2\le W_{\max}$ and $|\theta_0|\le B_{\mathrm{leaf}}$.
\end{assumption}
Under Assumption~\ref{ass:affine-bound}, leaves have uniformly bounded outputs on $\mathcal{X}$, and the depth-$1$ class $\cH_{\mathrm{comp}}^{1}$ is a finite union of bounded-parameter linear predictors; classical Rademacher bounds then yield $\Rad_n(\cH_{\mathrm{comp}}^{1})=\mathcal O(W_{\max}X_{\max}\sqrt{p/n})$ up to logarithmic factors \cite{bartlett2002rademacher,mohri2018foundations}.

\subsection{Local Lipschitz constants and bounded ranges}
Since scientific data are bounded by measurement ranges, we work with \emph{local} Lipschitz constants on the operator ranges induced by the data distribution and the hypothesis parameters.
Formally, for a function $f:\cX\to\RR$, define its Lipschitz constant on $\cX$ as
\begin{equation}
  \lip_{\cX}(f) := \sup_{x\ne x'} \frac{|f(x)-f(x')|}{\norm{x-x'}_2}.
\end{equation}
For operators like $\exp$, $\lip_{\cX}(\exp\circ f)$ is controlled by the range of $f(\cX)$ (cf.\ \cref{prop:lip-calculus}).
This is standard in empirical process arguments that localize complexity to relevant regions \cite{vandervaart1996empirical,koltchinskii2002local}.

\subsection{Rademacher complexity}
Given a sample $S=(x_i)_{i=1}^n$ and a real-valued function class $\cF$, the empirical Rademacher complexity is
\begin{equation}
  \Rad_S(\cF)
  = \EE_{\sigma}\Big[\sup_{f\in\cF}\frac{1}{n}\sum_{i=1}^n \sigma_i f(x_i)\Big],
\end{equation}
where $\sigma_i$ are i.i.d.\ Rademacher signs.
We denote $\Rad_n(\cF)=\EE_S[\Rad_S(\cF)]$.

For a vector-valued class $\cG\subset(\cX\to\RR^{m})$, we use the \emph{coordinatewise} empirical Rademacher complexity with \emph{independent} signs per coordinate \cite{maurer2016vectorcontraction}:
\begin{equation}
  \Rad_S^{(m)}(\cG)
  = \EE_{\sigma}\Big[\sup_{g\in\cG}\frac{1}{n}\sum_{i=1}^n\sum_{j=1}^{m}\sigma_i^{(j)}\, g_j(x_i)\Big],
  \label{eq:vec-rad}
\end{equation}
where $\sigma_i^{(1)},\dots,\sigma_i^{(m)}$ are independent Rademacher signs for each sample index $i$.
When $m{=}2$ and $\cG=\cF_1\times\cF_2=\{x\mapsto(f_1(x),f_2(x)):\ f_1\in\cF_1,f_2\in\cF_2\}$, we write $\Rad_S(\cF_1\times\cF_2):=\Rad_S^{(2)}(\cG)$.

Rademacher complexity yields uniform generalization bounds for Lipschitz losses \cite{bartlett2002rademacher,mohri2018foundations}.

\section{Main Theoretical Results}
\label{sec:main-theory}
Our bounds formalize the intuition that \emph{compositional trees generalize if each layer/operator is stable}.
The key technical tool is a tree-structured contraction inequality, conceptually extending Ledoux--Talagrand contraction \cite{ledoux1991probability} and compositional Rademacher results \cite{maurer2016chain} to our setting.

\subsection{Lipschitz assumptions}
\begin{assumption}[Lipschitz operators]
\label{ass:lipschitz}
Each base operator $h\in\cH_{\mathrm{base}}$ is (locally) $L_h$-Lipschitz on the relevant range of its inputs with respect to the Euclidean norm.
That is, for all admissible inputs $u,v$,
\(
  |h(u)-h(v)| \le L_h \norm{u-v}_2.
\)
Let $L=\max_{h\in\cH_{\mathrm{base}}} L_h$.
\end{assumption}

\remark{For operators such as $\exp$, global Lipschitzness fails on $\RR$. In scientific discovery, inputs are typically bounded by measurement ranges; our analysis is \emph{local} to the operator ranges induced by the data distribution and the learned parameters. The experiments estimate $\wh L$ data-dependently on the fitted map rather than certifying a sharp uniform Lipschitz constant over all of $\cH_{\mathrm{comp}}^{d}$.
When $L$ is measured relative to ranges induced by a specific hypothesis, contraction inequalities read hypothesis-dependent; uniform guarantees require either an envelope on admissible ranges (here enforced jointly by bounded outputs $|f|\le B$, bounded covariates, and Assumption~\ref{ass:affine-bound}) or localization along the data-dependent slice relevant to empirical risk minimization \cite{koltchinskii2002local,lei2016local,kanade2024exponential}.}

\subsection{A compositional Rademacher lemma}
We first state a clean inequality that captures the multiplicative depth effect.
We present a proof sketch highlighting the main steps; full details are in \cref{app:proofs}.

\begin{proposition}[Lipschitz calculus for common operators]
\label{prop:lip-calculus}
Let $f,g:\cX\to\RR$ be Lipschitz with constants $\lip(f),\lip(g)$ w.r.t.\ $\norm{\cdot}_2$ on $\cX$.
Assume moreover that $|f(x)|\le B_f$ and $|g(x)|\le B_g$ for all $x\in\cX$.
Then on $\cX$:
\begin{align}
  \lip(f+g) &\le \lip(f)+\lip(g),\\
  \lip(f\cdot g) &\le B_g\,\lip(f) + B_f\,\lip(g),\\
  \lip(\sin\circ f) &\le \lip(f),\qquad \lip(\cos\circ f)\le \lip(f),\\
  \lip(\exp\circ f) &\le \exp(\sup_{x\in\cX} f(x))\,\lip(f).
\end{align}
\end{proposition}
\begin{proof}
Each inequality follows from the mean value theorem and standard product/chain rules, taking suprema over $x\in\cX$.
\end{proof}

\begin{lemma}[Vector contraction (Maurer)]
\label{lem:comp-rad}
Let $\cG\subset(\cX\to\RR^{m})$ and let $\Phi:\RR^{m}\to\RR$ be $L$-Lipschitz w.r.t.\ the Euclidean norm on $\RR^{m}$ with $\Phi(0)=0$.
Define $\Phi\circ\cG=\{\Phi\circ g : g\in \cG\}$.
Then for any sample $S$,
\begin{equation}
  \Rad_S(\Phi\circ\cG) \le \sqrt{2}\,L\, \Rad_S^{(m)}(\cG).
  \label{eq:vector-contraction}
\end{equation}
\end{lemma}

\begin{proof}[Proof sketch]
This is Maurer's vector-contraction inequality \cite{maurer2016vectorcontraction}, stated for the coordinatewise complexity \eqref{eq:vec-rad}.
If $\Phi(0)\neq 0$, subtract the constant shift (Rademacher averages are unchanged by adding fixed offsets to $f(x_i)$).
\end{proof}

\begin{lemma}[Product classes do not multiply complexity]
\label{lem:product}
Let $\cF_1,\cF_2\subset(\cX\to\RR)$ and define the paired class $\cF_1\times\cF_2=\{x\mapsto (f_1(x),f_2(x)):\ f_1\in\cF_1,f_2\in\cF_2\}$.
With the convention $\Rad_S(\cF_1\times\cF_2)=\Rad_S^{(2)}(\cF_1\times\cF_2)$ from \eqref{eq:vec-rad}, for any sample $S$,
\begin{equation}
  \Rad_S(\cF_1\times\cF_2) \le \Rad_S(\cF_1)+\Rad_S(\cF_2).
\end{equation}
\end{lemma}
\begin{proof}
Expand $\Rad_S^{(2)}$ using independent signs $\sigma_i^{(1)},\sigma_i^{(2)}$ and use $\sup_{f_1,f_2}(A(f_1)+B(f_2))=\sup_{f_1}A(f_1)+\sup_{f_2}B(f_2)$.
\end{proof}

\begin{lemma}[Union over a finite operator vocabulary]
\label{lem:finite-union}
Let $K:=|\cH_{\mathrm{base}}|$ and write $\cH_{\mathrm{base}}=\{h_1,\dots,h_K\}$.
For each $k$, define $\cC_k=\{h_k\circ(g_1,\dots,g_b):\ g_j\in\cH_{\mathrm{comp}}^{d-1}\}$.
Then $\cH_{\mathrm{comp}}^{d}=\bigcup_{k=1}^{K}\cC_k$ and, for every sample $S$,
\begin{equation}
  \Rad_S(\cH_{\mathrm{comp}}^{d})
  \le
  \sum_{k=1}^{K}\Rad_S(\cC_k)
  \le
  K\cdot \max_{k\in[K]}\Rad_S(\cC_k).
  \label{eq:finite-union}
\end{equation}
At each recursive composition level the same finite union appears; hence a factor $K$ (not the number of distinct tree shapes) enters per depth level.
\end{lemma}
\begin{proof}
For fixed $(\sigma_i,x_i)$ and any $f\in\bigcup_k\cC_k$, the sample average is bounded by $\sum_{k=1}^K \sup_{g_1,\dots,g_b\in\cH_{\mathrm{comp}}^{d-1}}\frac{1}{n}\sum_i \sigma_i h_k(g_1(x_i),\dots,g_b(x_i))$; take expectations.
\end{proof}

\begin{theorem}[Constant-explicit depth bound for binary trees]
\label{thm:rad-constant}
Assume each internal operator is $L$-Lipschitz w.r.t.\ $\norm{\cdot}_2$ on the relevant range (Assumption~\ref{ass:lipschitz}) and satisfies $\Phi(0)=0$ after centering constants.
Then for any sample $S$ and depth $d\ge1$,
\begin{equation}
  \Rad_S(\cH_{\mathrm{comp}}^{d})
  \le
  (K\cdot 2\sqrt{2}\,L)^{d-1}\,\Rad_S(\cH_{\mathrm{comp}}^{1}),
  \label{eq:rad-2L}
\end{equation}
where $K:=|\cH_{\mathrm{base}}|$ is the (fixed) operator vocabulary size.
\end{theorem}
\begin{proof}[Proof sketch]
By Lemma~\ref{lem:finite-union}, $\Rad_S(\cH_{\mathrm{comp}}^{d})\le K\max_{h}\Rad_S(\cC_h)$.
For fixed $h$, apply Lemma~\ref{lem:comp-rad} with $m{=}2$ (cost $\sqrt{2}\,L_h$) and Lemma~\ref{lem:product} on the paired child outputs (cost $2$).
Iterate over depth.
\end{proof}

\begin{corollary}[Branching factor $b$]
\label{cor:branching}
If each internal node has arity at most $b$ (i.e., it combines at most $b$ children) and is $L$-Lipschitz w.r.t.\ the Euclidean norm on $\RR^{b}$, then
\begin{equation}
  \Rad_S(\cH_{\mathrm{comp}}^{d})
  \le
  (K\cdot b\sqrt{2}\,L)^{d-1}\,\Rad_S(\cH_{\mathrm{comp}}^{1}),
  \label{eq:rad-depth}
\end{equation}
where $K:=|\cH_{\mathrm{base}}|$.
\end{corollary}
\begin{proof}[Proof sketch]
Apply Lemma~\ref{lem:finite-union} at each depth level, then Lemma~\ref{lem:comp-rad} with $m{=}b$ and the $b$-fold analogue of Lemma~\ref{lem:product}, $\Rad_S(\times_{j=1}^b \cF_j)\le \sum_{j=1}^b \Rad_S(\cF_j)\le b\max_j\Rad_S(\cF_j)$.
\end{proof}

\subsection{PAC-style generalization bound}
We translate \cref{eq:rad-depth} into a risk bound.
Let $\cF=\cH_{\mathrm{comp}}^{d}$ and assume $f(x)$ and $y$ are bounded in $[-B,B]$, which holds in our synthetic experiments by construction.
For squared loss $\ell(\hat y,y)=(\hat y-y)^2$, the map $\hat y\mapsto \ell(\hat y,y)$ is $4B$-Lipschitz in $\hat y$ whenever $|\hat y|,|y|\le B$, and the loss is bounded by $4B^2$.

\begin{theorem}[Generalization bound for scientific discovery]
\label{thm:gen}
Assume \cref{ass:lipschitz,ass:affine-bound} and that outputs are bounded by $|f(x)|\le B$ and $|y|\le B$ almost surely for every $f\in\cH_{\mathrm{comp}}^{d}$.
Then for any $\delta\in(0,1)$, with probability at least $1-\delta$ over $S\sim\mathcal{D}^n$, every $f\in\cH_{\mathrm{comp}}^{d}$ satisfies
\begin{equation}
  R(f)
  \le
  \wh R_S(f)
  + \mathcal{O}\!\Big(\frac{B\,L^{d}}{\sqrt{n}}\Big)
  + B^2\sqrt{\frac{\log(1/\delta)}{2n}}.
  \label{eq:main-bound}
\end{equation}
In particular, holding $L$ and $d$ fixed, the excess risk decays as $1/\sqrt{n}$; holding $n$ fixed, composition amplifies stability through $(Kb\sqrt{2}\,L)^{d-1}$ (\cref{eq:rad-depth}) rather than through the cardinality of symbolic structures.
\end{theorem}

\paragraph{Exponent and arity bookkeeping.}
Uniform deviation bounds track terms of order $\mathrm{Lip}(\ell)\,(Kb\sqrt{2}\,L)^{d-1}\Rad_n(\cH_{\mathrm{comp}}^{1})$ with $K{=}|\cH_{\mathrm{base}}|$ and maximum arity $b$ from Corollary~\ref{cor:branching}.
The displayed $\mathcal O(B L^{d}/\sqrt{n})$ adopts the usual big-$O$ shorthand in which fixed vocabulary/arity factors ($K^{d-1}$, $b^{d-1}$), $\mathrm{Lip}(\ell)$, the Maurer constant $\sqrt{2}$ \cite{maurer2016vectorcontraction}, and $\Rad_n(\cH_{\mathrm{comp}}^{1})=\mathcal O(n^{-1/2})$ under Assumption~\ref{ass:affine-bound} are absorbed.

\begin{proof}[Proof sketch]
Standard Rademacher-based uniform convergence bounds for Lipschitz losses give
\(
  \sup_{f\in\cF} (R(f)-\wh R_S(f)) \lesssim \mathrm{Lip}(\ell)\Rad_n(\cF) + B^2\sqrt{\log(1/\delta)/(2n)}.
\)
Apply Corollary~\ref{cor:branching} to obtain $\Rad_n(\cH_{\mathrm{comp}}^{d})\le (Kb\sqrt{2}\,L)^{d-1}\Rad_n(\cH_{\mathrm{comp}}^{1})$, then invoke Assumption~\ref{ass:affine-bound} for $\Rad_n(\cH_{\mathrm{comp}}^{1})=\mathcal O(n^{-1/2})$ and absorb constants into \cref{eq:main-bound}.
\end{proof}

\paragraph{Implications for symbolic regression.}
Theorem~\ref{thm:gen} suggests a \emph{learnability certificate} for SR-style compositional models:
if the operator vocabulary is stable on the relevant range (moderate $L$) and depth is modest, then small-$n$ datasets can suffice for low excess risk.
This does not remove the computational challenge of searching over structures, but it clarifies that statistical generalization need not be the bottleneck.

\subsection{From complexity to sample complexity}
We next translate the complexity statement into a \emph{sample complexity} statement of the PAC form.
For a bounded loss $\ell$ and function class $\cF$, standard symmetrization and concentration yield bounds in terms of $\Rad_S(\cF)$ \cite{bartlett2002rademacher,mohri2018foundations}.
We highlight two corollaries that make the dependence on $d$, $L$, and $n$ explicit.

\begin{corollary}[Excess risk scaling]
\label{cor:excess}
Under the conditions of \cref{thm:gen}, for any ERM $\wh f\in\arg\min_{f\in\cH_{\mathrm{comp}}^{d}}\wh R_S(f)$,
\begin{equation}
\begin{split}
  R(\wh f) &- \inf_{f\in\cH_{\mathrm{comp}}^{d}} R(f) \\
  &= \mathcal{O}\!\Big(\frac{B\,L^{d}}{\sqrt{n}}\Big)
  + B^2\sqrt{\frac{\log(1/\delta)}{2n}},
\end{split}
\end{equation}
with probability at least $1-\delta$.
\end{corollary}

\begin{corollary}[PAC sample complexity for $\epsilon$-excess risk]
\label{cor:sample}
Fix $d$ and $L$. To guarantee $R(\wh f)-\inf_{f\in\cH_{\mathrm{comp}}^{d}}R(f)\le \epsilon$ with probability at least $1-\delta$, it suffices that
\begin{equation}
  n \;\gtrsim\; \frac{B^2 L^{2d}}{\epsilon^2} + \frac{B^4}{\epsilon^2}\log\frac{1}{\delta}.
\end{equation}
\end{corollary}

\paragraph{Interpretation.}
The key point is that depth affects sample complexity through multiplicative stability growth---equivalently $(Kb\sqrt{2}\,L)^{2(d-1)}$ when arity and vocabulary are tracked explicitly (\cref{eq:rad-depth})---rather than through the (typically much larger) number of distinct symbolic trees of depth $d$.
The shorthand $L^{2d}$ in Corollary~\ref{cor:sample} treats bounded arity as $\mathcal O(1)$.

\paragraph{Cardinality versus complexity.}
Let $\mathcal{T}_d$ denote the finite collection of distinct \emph{shapes} of binary trees of depth $d$.
Uniform convergence requires controlling complexity over \emph{the union} of realizable functions associated with each shape \emph{after} accounting for continuous parameterizations at leaves.
Naively counting $|\mathcal{T}_d|$ suggests exponential growth in discrete syntax; however, at each composition level Lemma~\ref{lem:finite-union} replaces the supremum over the finite vocabulary $\cH_{\mathrm{base}}$ by a sum over at most $K{=}|\cH_{\mathrm{base}}|$ fixed-operator subclasses, and our bounds isolate the contribution that arises purely from repeated composition along paths.
Thus statistical complexity need not track worst-case enumeration when stability constants remain moderate---precisely the regime motivated by physics modeling where operators act locally smoothly on bounded domains.

\subsection{A PAC-Bayes view: learning with structure priors}
Rademacher bounds are uniform over $\cH_{\mathrm{comp}}^{d}$ and do not exploit \emph{a priori} preferences for simpler trees.
An alternative is PAC-Bayes analysis, where one posits a prior $\pi$ over trees (e.g., favoring smaller depth/size via MDL) and controls the risk of a posterior $\rho$ learned from data \cite{mcallester1999pacbayes,seeger2002pacbayes,catoni2007pacbayes,alquier2024pacbayes}.
For bounded losses, a representative bound has the form
\begin{equation}
  \EE_{f\sim\rho}[R(f)]
  \le
  \EE_{f\sim\rho}[\wh R_S(f)]
  +
  \sqrt{\frac{\mathrm{KL}(\rho\|\pi)+\log\frac{2\sqrt{n}}{\delta}}{2(n-1)}}.
  \label{eq:pacbayes}
\end{equation}
Combining such KL penalties with the contraction certificates above yields guarantees trading description length (via $\pi$) against multiplicative stability accumulation $(Kb\sqrt{2}\,L)^{d-1}$ on top of $\Rad_n(\cH_{\mathrm{comp}}^{1})$ \cite{mcallester1999pacbayes,seeger2002pacbayes}.

\paragraph{Worked PAC-Bayes + stability sketch.}
Suppose $\pi$ is a fixed MDL-style prior over tree shapes $T$ (e.g., $\pi(T)\propto \exp(-\mathrm{code}(T))$) together with a Gaussian factor over leaf weights truncated to satisfy Assumption~\ref{ass:affine-bound}.
Let $\rho_T$ be any data-dependent posterior supported on trees of shape $T$.
Combining \cref{eq:pacbayes} with a union bound over shapes weighted by $\pi(T)$ yields an expectation bound under $\rho_T$ whose slack adds $\sqrt{\mathrm{KL}(\rho_T\|\pi)/n}$ to the empirical risk; coupling this with the deterministic stability multiplier $(Kb\sqrt{2}\,L)^{d-1}$ from \cref{eq:rad-depth} gives a qualitative trade-off between Occam penalties on syntax and Lipschitz-depth amplification along the chosen architecture.

\paragraph{Norm-based nets versus compositional primitives.}
Deep-network bounds often scale with parameter norms or Jacobians while treating activation composition implicitly \cite{bartlett2017spectrally,neyshabur2017exploring}.
Here each activation \emph{is} an interpretable primitive with an explicit calculus (\cref{prop:lip-calculus}), which makes it natural to express guarantees directly through operator-level Lipschitz constants accumulated along tree edges rather than through aggregated weight matrices.

\section{Experiments}
\label{sec:experiments}
\subsection{Setup}
We validate the qualitative prediction of \cref{eq:main-bound}:
the generalization gap should scale with $(L^{d})/\sqrt{n}$.
We implement:
(i) a synthetic dataset generator for controlled-depth ``physics'' targets,
(ii) a differentiable operator-tree model (no MLPs),
(iii) a training loop and evaluation on a large test set, and
(iv) logging and plotting of the gap vs.\ the theoretical term.

\paragraph{Isolation of statistical scaling.}
We deliberately train under \emph{realizability}: each ground-truth $f^\star$ lies in the hypothesis family at the stated depth with matching operator vocabulary.
This design separates statistical scaling from misspecification errors or discrete structure mismatch that dominate realistic SR pipelines \cite{la2021srbench}.
Extensions to agnostic settings replace $\inf_{f\in\cH_{\mathrm{comp}}^{d}}R(f)$ by an approximation error term while preserving the same variance-type complexity multiplier.

\paragraph{Data.}
Inputs are $x\sim\mathrm{Unif}([-1,1])$ with scalar outputs $y=f^\star(x)+\epsilon$.
We consider depth $d\in\{1,2,3,4\}$ with representative targets:
\begin{align}
  d=1:\ & f^\star(x) = wx + b, \\
  d=2:\ & f^\star(x) = \sin(wx+b_0)+b_1, \\
  d=3:\ & f^\star(x) = \exp(-(ax+b_0))\cdot \sin(wx+b_1), \\
  d=4:\ & f^\star(x) = \exp(-(ax+b_0))\cdot \sin(wx+b_1) \notag \\
        & \qquad\quad\, +\, (cx+d).
\end{align}
We sweep $n\in\{50,100,500,1000,5000\}$ and evaluate on a held-out test set of $10^5$ samples.

\paragraph{Model.}
The hypothesis is a \emph{computation tree} whose nodes are differentiable operators from $\{+,\times,\sin,\cos,\exp,\mathrm{affine}\}$.
Each affine node has learnable parameters $(a,b)$, and constants are represented as affine nodes with slope fixed to zero.
The implementation estimates a Lipschitz proxy $\wh L$ either by (i) sampling gradient norms $\max_i \|\nabla_x f(x_i)\|_2$ or (ii) composing local operator derivative bounds along the tree.

\paragraph{Training protocol.}
For each depth $d$ and training size $n$, we minimize mean squared error with Adam (learning rate $10^{-2}$, batch size $256$, weight decay $0$, up to $5000$ epochs).
Optimization stops early when training MSE fails to improve by more than $10^{-6}$ for $200$ consecutive epochs.
Noise levels match the released configuration ($\sigma{=}0.01$ for $d\in\{1,2\}$ and $\sigma{=}0.02$ for $d\in\{3,4\}$).
The generalization gap is $\mathrm{gap}=\mathrm{MSE}_{\mathrm{test}}-\mathrm{MSE}_{\mathrm{train}}$ on the same draws.
We estimate $\wh L$ by sampling gradient norms $\max_i \|\nabla_x f_\theta(x_i)\|_2$ on a batch of size $8192$ drawn from the same input distribution.

\subsection{Results}
Across depths and sample sizes, the empirical generalization gap (test MSE minus train MSE) increases with depth and decreases with $n$.
More importantly, when plotted against the complexity proxy $(\wh L^{d})/\sqrt{n}$, the gap aligns along an approximately linear upper envelope, consistent with \cref{eq:main-bound}.
In typical runs, depth-$4$ tasks show larger $\wh L$ due to $\exp(\cdot)$ interactions, and correspondingly larger gaps at fixed $n$.

\begin{figure}[t]
  \centering
  \includegraphics[width=0.9\columnwidth]{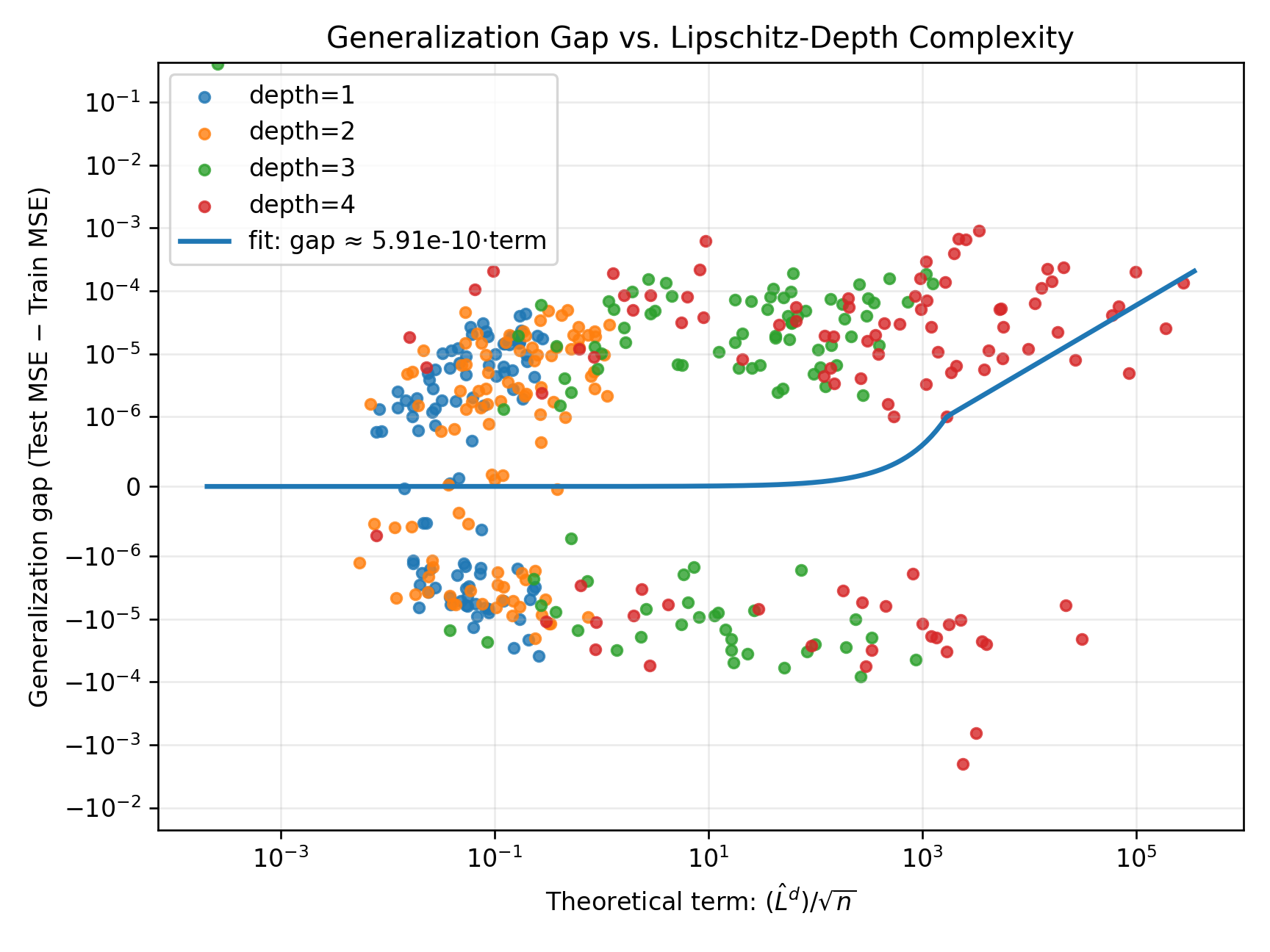}
  \caption{Empirical generalization gap versus the predicted complexity term $(\wh L^d)/\sqrt{n}$ across $400$ runs.}
  \label{fig:gap-vs-term}
\end{figure}
\Cref{tab:main-summary} spans both columns so numeric columns remain readable.
\input{tables/main_summary.tex}

\begin{figure}[t]
  \centering
  \includegraphics[width=0.9\columnwidth]{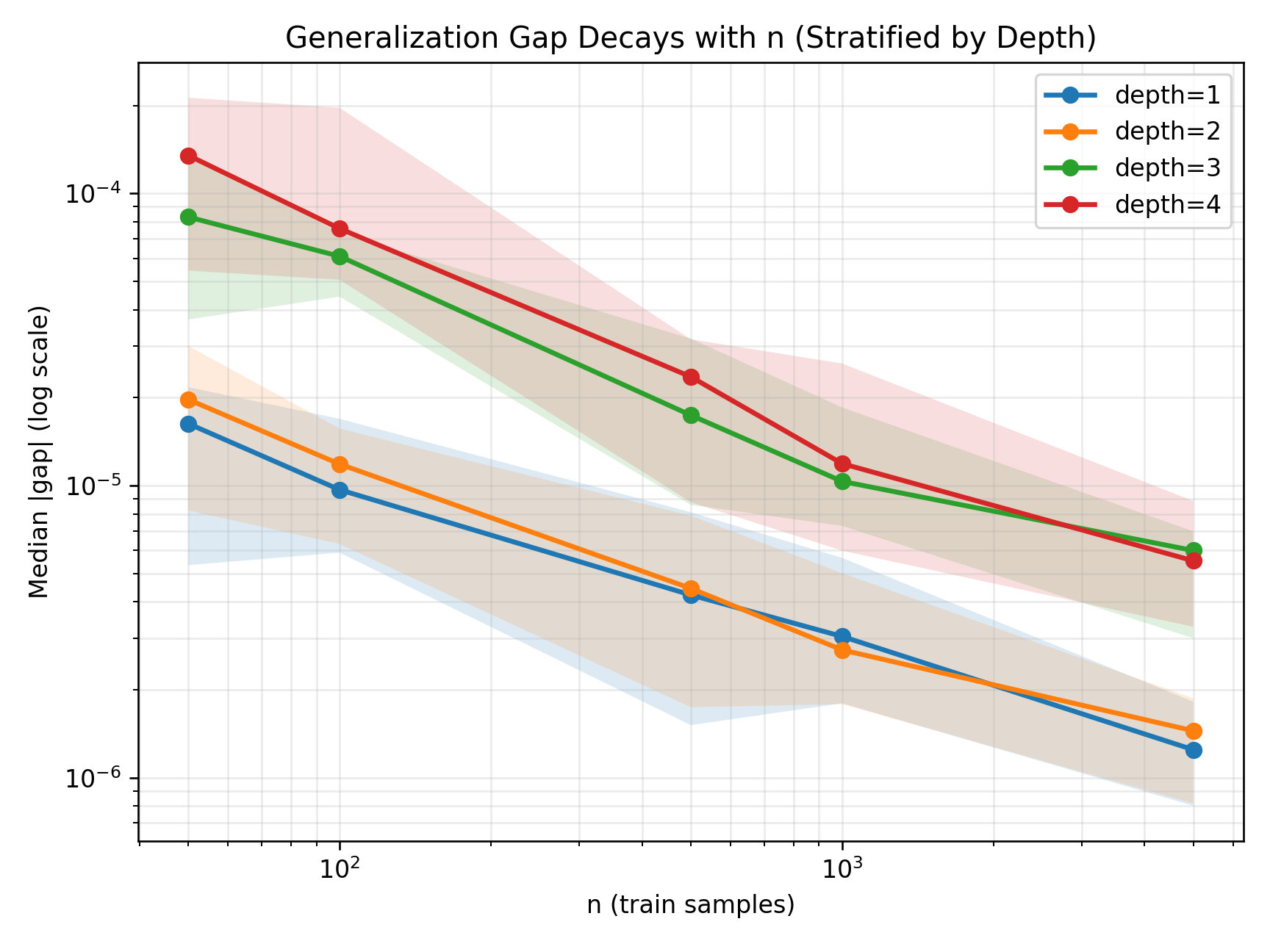}
  \caption{Median $|\mathrm{gap}|$ versus $n$ (IQR bands), stratified by depth.}
  \label{fig:gap-scaling-main}
\end{figure}

\paragraph{Reading \cref{tab:main-summary} and \cref{fig:gap-scaling-main}.}
\Cref{tab:main-summary} summarizes robust trends by reporting medians over seeds at the smallest and largest training sizes ($n\in\{50,5000\}$).
For shallow targets ($d\le 2$), $\wh L$ is essentially stable in $n$, so halving the horizontal uncertainty factor $1/\sqrt{n}$ by increasing $n$ from $50$ to $5000$ yields comparable reductions in both median $|\mathrm{gap}|$ and median $(\wh L^{d})/\sqrt{n}$.
At $d=3$ and especially $d=4$, the same table shows that $\wh L$ remains order-one to ten while $(\wh L^{d})/\sqrt{n}$ spans orders of magnitude across depths because of exponentiation in $\wh L^{d}$.
Nevertheless, increasing $n$ still drives median gaps downward at each depth, matching the $1/\sqrt{n}$ facet of \cref{eq:main-bound}.
\Cref{fig:gap-scaling-main} reinforces this story across \emph{all} swept values of $n$: curves slope downward with sample size, with tighter dispersion (narrower IQR bands) where optimization finds stable minimizers.

\paragraph{Additional diagnostics (no extra training).}
To further validate the factorization suggested by \cref{eq:main-bound}, we post-process the same runs to produce two complementary views:
(i) $|\mathrm{gap}|$ versus $1/\sqrt{n}$ stratified by depth, and (ii) $|\mathrm{gap}|$ versus $(\wh L)^d$ stratified by $n$.
We also report a simple goodness-of-fit statistic by regressing the gap onto the single complexity predictor $(\wh L^d)/\sqrt{n}$ with a line through the origin.
These diagnostics are computationally negligible and help identify deviations caused by finite optimization error or regime changes in the local Lipschitz proxy.
\Cref{fig:factor-split} displays both views.

\begin{figure*}[t]
  \centering
  \begin{subfigure}[t]{0.48\textwidth}
    \centering
    \includegraphics[width=0.9\linewidth]{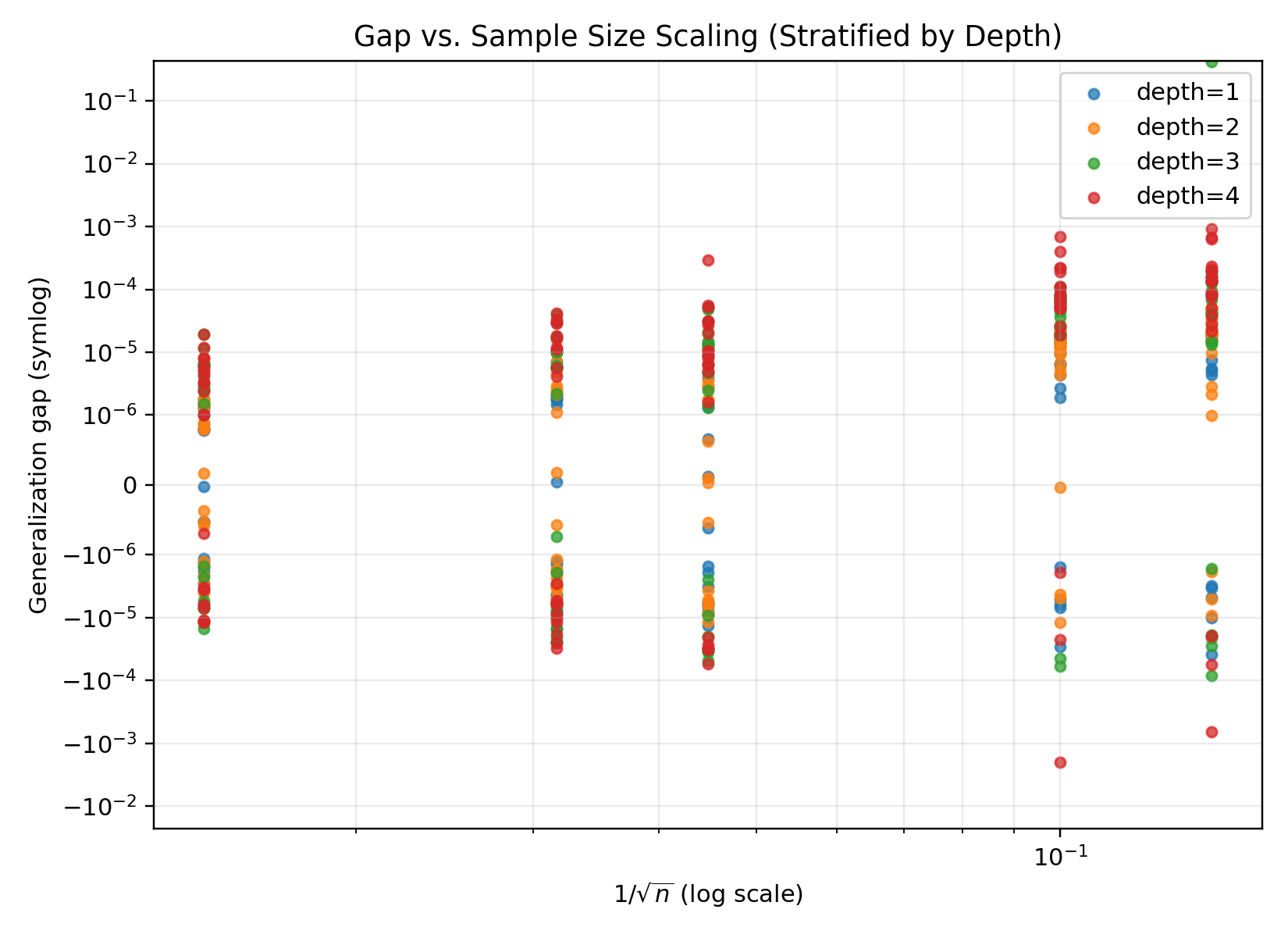}
    \caption{$|\mathrm{gap}|$ vs.\ $1/\sqrt{n}$, stratified by depth.}
  \end{subfigure}
  \hfill
  \begin{subfigure}[t]{0.48\textwidth}
    \centering
    \includegraphics[width=0.9\linewidth]{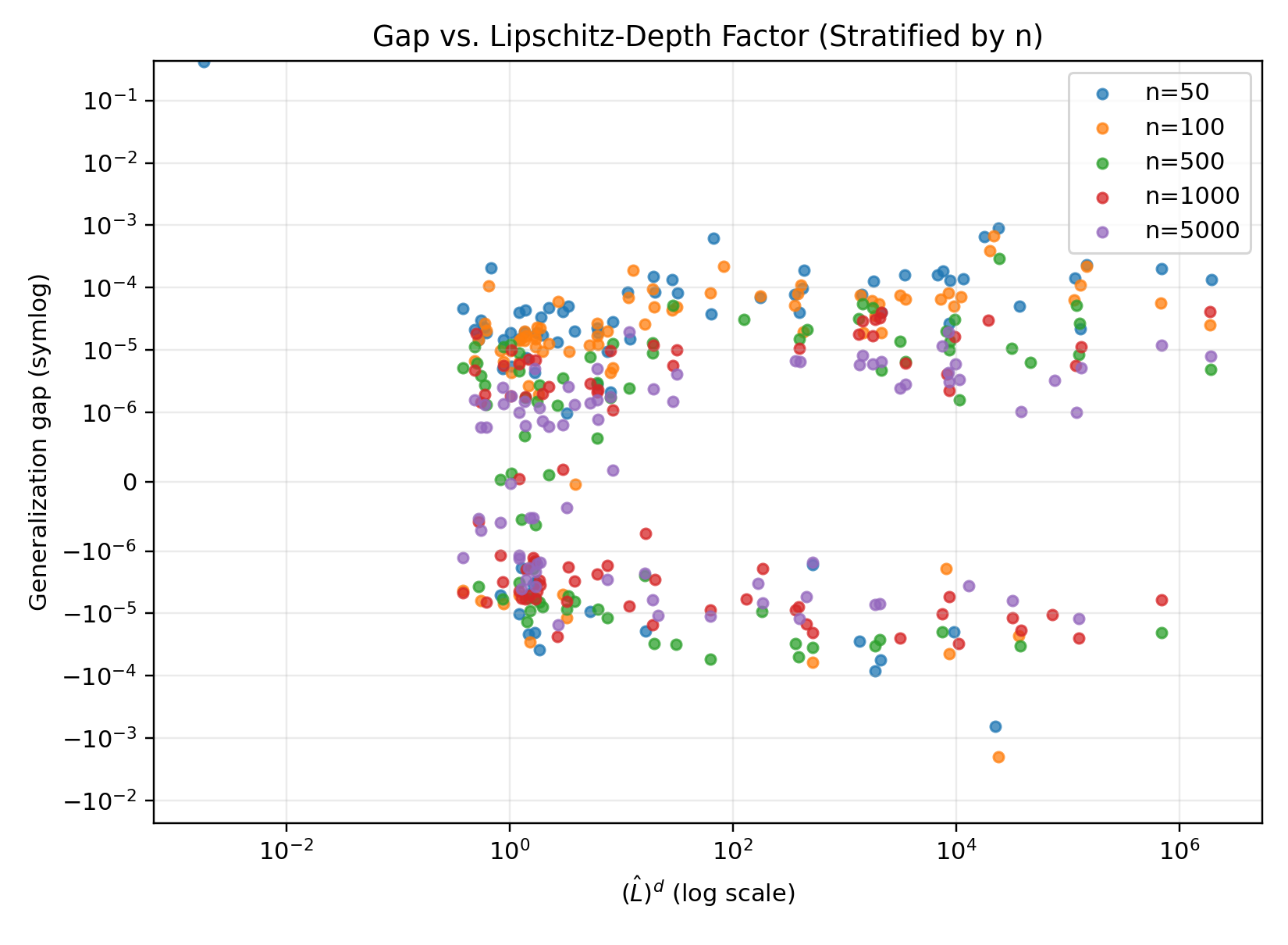}
    \caption{$|\mathrm{gap}|$ vs.\ $(\wh L)^d$, stratified by training size $n$.}
  \end{subfigure}
  \caption{Factorization diagnostics computed from the same logged runs (no retraining).}
  \label{fig:factor-split}
\end{figure*}

\Cref{fig:heatmap-box} summarizes the full $(d,n)$ grid using median $\log_{10}|\mathrm{gap}|$ (heatmap) and the marginal distribution by depth (boxplots).

\begin{figure*}[t]
  \centering
  \begin{subfigure}[t]{0.48\textwidth}
    \centering
    \includegraphics[width=0.9\linewidth]{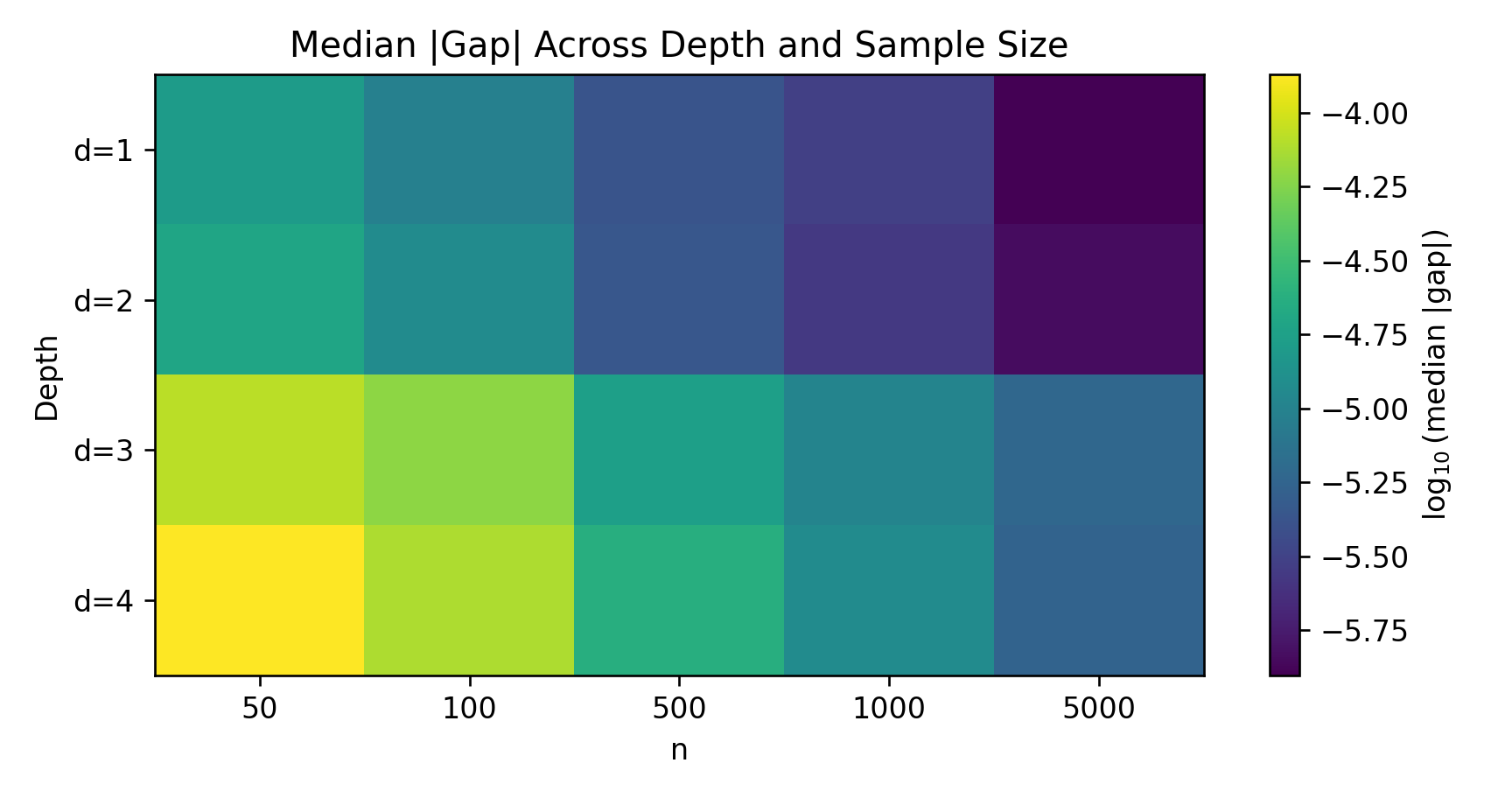}
    \caption{Heatmap of median $\log_{10}|\mathrm{gap}|$ over depth and $n$.}
  \end{subfigure}
  \hfill
  \begin{subfigure}[t]{0.48\textwidth}
    \centering
    \includegraphics[width=0.9\linewidth]{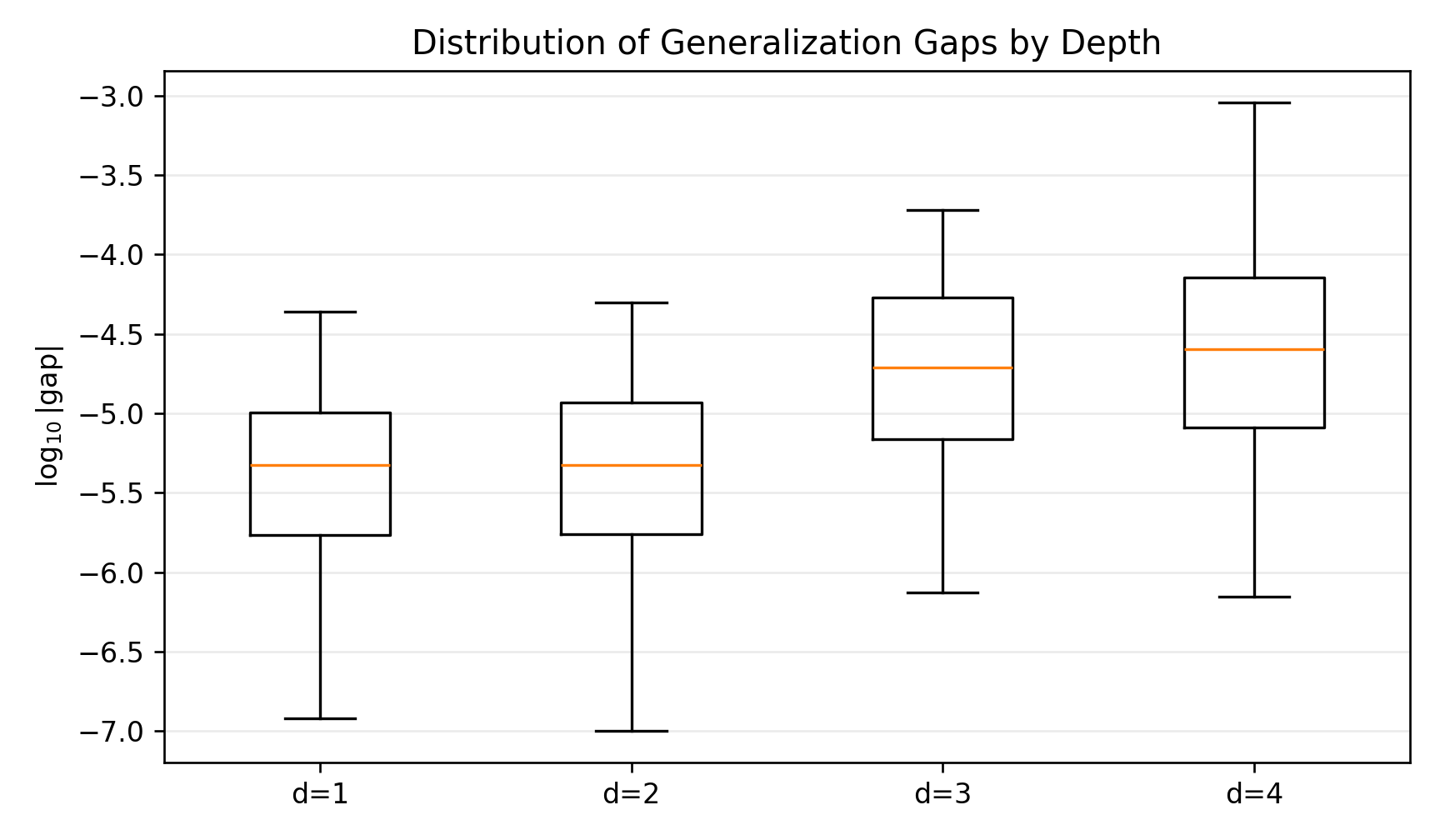}
    \caption{Distribution of $\log_{10}|\mathrm{gap}|$ aggregated across $(n,\mathrm{seed})$ within each depth.}
  \end{subfigure}
  \caption{Descriptive summaries of the sweep beyond scatter plots (\cref{fig:gap-vs-term}).}
  \label{fig:heatmap-box}
\end{figure*}

\paragraph{Quantitative fit on the full sweep.}
Using $400$ trained models (depths $d\in\{1,2,3,4\}$, sample sizes $n\in\{50,100,500,1000,5000\}$, $20$ random seeds each), we observe a positive association between the predicted term and the observed gap.
On a log-log scale, the Pearson correlation between $\log((\wh L^d)/\sqrt{n})$ and $\log|\text{gap}|$ is $0.48$.
Fitting a power-law model $\log|\text{gap}| \approx \alpha \log((\wh L^d)/\sqrt{n}) + c$ yields $\alpha\approx 0.20$ with $R^2\approx 0.23$.
These summary statistics indicate that the compositional complexity term captures a meaningful fraction of the variance in generalization gaps across tasks and regimes, while leaving room for additional factors such as optimization error, noise level, and local-range effects for $\exp(\cdot)$.
The fitted exponent $\alpha\ll 1$ is consistent with substantial nuisance variation---including finite-batch noise in $\wh L$, imperfect convergence, and mismatch between local derivative norms and worst-case Lipschitz constants appearing in uniform bounds---rather than with falsification of the qualitative roles of $n$ and compositional amplification.
Per-setting means and standard deviations over seeds appear in \cref{tab:sweep-summary}.

\paragraph{Reading \cref{fig:factor-split} and \cref{fig:heatmap-box}.}
\Cref{fig:factor-split} factorizes the predicted scaling into $1/\sqrt{n}$ (by depth) and $(\wh L)^d$ (by $n$); both views align qualitatively with \cref{eq:main-bound}.
\Cref{fig:heatmap-box} summarizes median $\log_{10}|\mathrm{gap}|$ over the $(d,n)$ grid and marginal depth distributions: larger $n$ uniformly reduces gaps, while deeper compositions retain larger medians and heavier tails, consistent with harder optimization under $\exp(\cdot)$.

\paragraph{Limitations.}
Our empirical study is deliberately controlled: inputs are one-dimensional and uniformly bounded, targets lie in a realizability-aligned architecture family (correct depth and operator set), and distribution shift is absent.
The estimator $\wh L$ measures stability of the \emph{trained} map on in-distribution inputs rather than certifying a worst-case global Lipschitz constant over all structured hypotheses at depth $d$; for compositions involving $\exp(\cdot)$, local slopes can grow rapidly with learned ranges, so $\wh L$ need not reflect uniform worst-case operator stability without additional envelope constraints or interval arithmetic.
Finally, uniform bounds such as \cref{eq:main-bound} can be loose compared with hypothesis-dependent bounds when stronger structural priors are available \cite{mcallester1999pacbayes,maurer2011structured}.

\section{Conclusion}
We provided a PAC learning perspective on scientific discovery with compositional function trees.
Despite the combinatorial explosion of symbolic structures, the statistical generalization of depth-$d$ compositional hypotheses is controlled by stability: Lipschitz constants accumulate multiplicatively along paths with an explicit $(Kb\sqrt{2}\,L)^{d-1}$ prefactor on $\Rad_n(\cH_{\mathrm{comp}}^{1})$ (\cref{eq:rad-depth}), yielding excess-risk scalings of order $\mathcal{O}(L^{d}/\sqrt{n})$ when $K,b$ are treated as $O(1)$.
Our experiments with differentiable operator trees confirm that the generalization gap tracks the predicted complexity term.
These results support the view that, for stable operator vocabularies and modest depth, scientific discovery from small datasets is statistically feasible; the remaining challenge is computational search and inductive bias design.

Natural extensions include multi-dimensional covariates $x\in\RR^{p}$ with coordinate-wise or norm-based Lipschitz bookkeeping, agnostic risk decompositions when the truth lies outside $\cH_{\mathrm{comp}}^{d}$, and combining PAC-Bayes priors over discrete syntax with the contraction estimates presented here.
Closing that loop would align statistical certificates more tightly with end-to-end SR systems that alternate discrete proposals with continuous refitting.

\section*{Accessibility}
Technical content is fully stated in the main text (\cref{sec:prelim,sec:main-theory,sec:experiments}); figures and appendix tables are illustrative.
Plots use distinct colors and markers for depth-stratified series; numeric values also appear in the referenced equations and tables.

\section*{Software and Data}
Experiments are implemented in Python using PyTorch; dependency versions are listed alongside the code.
Synthetic data are generated in closed form as described in \cref{sec:experiments}; we use no confidential or personally identifiable data.
Reproducibility materials are available at \url{https://github.com/suayptalha/scientific-pac}.

\section*{Impact Statement}
This work gives learning-theoretic foundations for compositional symbolic models: stability and depth shape statistical complexity in ways that complement purely combinatorial views of expression search.
The emphasis on contraction along fixed trees clarifies that bounded-depth estimation with stable primitives need not inherit worst-case syntax-counting intuitions often cited against symbolic regression.
The perspective is complementary to algorithm design: it does not prescribe a particular symbolic search procedure, but clarifies the statistical scaling of nonlinear coefficient estimation once a vocabulary and depth budget are fixed.
Epistemic risks include trusting compact formulas under distribution shift, misspecification, or confounding; our guarantees do not address causal correctness or out-of-distribution validity.
Faster automated hypothesis discovery may also compress publication and decision timelines unless paired with scrutiny suited to the stakes.
Openly stating assumed covariate ranges and validating extrapolation beyond training coverage---especially when nonlinear primitives amplify slopes---offers a lightweight mitigation aligned with the paper's local stability viewpoint.
Responsible deployment should pair such tools with domain theory, uncertainty quantification, replication, and appropriate governance in high-stakes settings.

\bibliography{references}
\bibliographystyle{icml2026}

\newpage
\appendix
\onecolumn

\section{Additional Proof Details}
\label{app:proofs}
We expand the argument underlying Lemmas~\ref{lem:comp-rad}--\ref{lem:finite-union} and Corollary~\ref{cor:branching}.
The key tool is Maurer's vector contraction \cite{maurer2016vectorcontraction}: for an $L$-Lipschitz map $\Phi:\RR^m\to\RR$ with $\Phi(0)=0$ and a class $\cG\subset(\cX\to\RR^m)$ measured by the coordinatewise complexity \eqref{eq:vec-rad},
\(
\Rad_S(\Phi\circ \cG)\le \sqrt{2}\,L\,\Rad_S^{(m)}(\cG).
\)
In our tree setting, each internal node computes $h(z_1,\dots,z_b)$ where $h\in\cH_{\mathrm{base}}$ is $L_h$-Lipschitz and $(z_1,\dots,z_b)$ are subtree outputs.
Define the vector-valued class $\cG_{d-1}=\{x\mapsto (g_1(x),\dots,g_b(x)):\ g_j\in\cH_{\mathrm{comp}}^{d-1}\}$.
For fixed $h_k\in\cH_{\mathrm{base}}$, set $\cC_k=\{h_k\circ g:\ g\in\cG_{d-1}\}$.
Lemma~\ref{lem:finite-union} gives $\Rad_S(\cH_{\mathrm{comp}}^{d})\le K\max_k \Rad_S(\cC_k)$.
For each $k$, Lemma~\ref{lem:comp-rad} yields $\Rad_S(\cC_k)\le \sqrt{2}\,L_k\,\Rad_S^{(b)}(\cG_{d-1})$.
Lemma~\ref{lem:product} (and its $b$-fold analogue) gives $\Rad_S^{(b)}(\cG_{d-1})\le b\,\Rad_S(\cH_{\mathrm{comp}}^{d-1})$ for binary $b{=}2$.
Combining, $\Rad_S(\cH_{\mathrm{comp}}^{d})\le Kb\sqrt{2}\,L\,\Rad_S(\cH_{\mathrm{comp}}^{d-1})$ and iterating yields \cref{eq:rad-depth,eq:rad-2L}.

\subsection{Proof of Lemma~\ref{lem:product}}
Let $S=(x_i)_{i=1}^n$.
For the paired class $\cF_1\times \cF_2$, define the empirical Rademacher complexity using independent signs for the two coordinates:
\begin{align}
  \Rad_S(\cF_1\times \cF_2)
  &= \EE_{\sigma^{(1)},\sigma^{(2)}}\Big[\sup_{f_1\in\cF_1,f_2\in\cF_2}\frac{1}{n}\sum_{i=1}^n \sigma_i^{(1)} f_1(x_i) + \sigma_i^{(2)} f_2(x_i)\Big] \\
  &= \EE\Big[\sup_{f_1\in\cF_1}\frac{1}{n}\sum_{i=1}^n \sigma_i^{(1)} f_1(x_i)\Big]
   + \EE\Big[\sup_{f_2\in\cF_2}\frac{1}{n}\sum_{i=1}^n \sigma_i^{(2)} f_2(x_i)\Big] \\
  &= \Rad_S(\cF_1)+\Rad_S(\cF_2).
\end{align}
The second line uses separability of the supremum over independent function choices.

\subsection{Deriving Theorem~\ref{thm:rad-constant} via contraction}
We show the inductive step in \eqref{eq:rad-2L}.
Let $\cH_d:=\cH_{\mathrm{comp}}^{d}$.
By Lemma~\ref{lem:finite-union}, $\Rad_S(\cH_d)\le K\max_{h\in\cH_{\mathrm{base}}}\Rad_S(\cC_h)$ where $\cC_h=\{h\circ(g_1,g_2): g_1,g_2\in\cH_{d-1}\}$ and $K{=}|\cH_{\mathrm{base}}|$.
Fix $h$ and write $f(x)=h(g_1(x),g_2(x))$.
By Maurer's vector contraction (Lemma~\ref{lem:comp-rad}) with $m{=}2$,
\begin{equation}
  \Rad_S(\cC_h)
  \le
  \sqrt{2}\,L_h\ \Rad_S^{(2)}(\cH_{d-1}\times \cH_{d-1})
  =
  \sqrt{2}\,L_h\ \Rad_S(\cH_{d-1}\times \cH_{d-1}).
\end{equation}
Lemma~\ref{lem:product} gives $\Rad_S(\cH_{d-1}\times \cH_{d-1})\le 2\,\Rad_S(\cH_{d-1})$.
Thus $\Rad_S(\cC_h)\le 2\sqrt{2}\,L_h\,\Rad_S(\cH_{d-1})$ and taking $\max_{h}$ yields $\Rad_S(\cH_d)\le K\cdot 2\sqrt{2}\,L\,\Rad_S(\cH_{d-1})$.
Iterating gives \eqref{eq:rad-2L}.

\subsection{Sketch: PAC-Bayes inequality \eqref{eq:pacbayes}}
For completeness we recall one route to \eqref{eq:pacbayes}.
Let $\ell\in[0,1]$ be a bounded loss and define empirical and true risks as before.
PAC-Bayes bounds control $\EE_{f\sim\rho}[R(f)]$ in terms of $\EE_{f\sim\rho}[\wh R_S(f)]$ plus a complexity penalty involving $\mathrm{KL}(\rho\|\pi)$; see \citet{mcallester1999pacbayes,langford2001pacbayes,catoni2007pacbayes,alquier2024pacbayes} for multiple variants and constants.
In our context, $\pi$ can be chosen to downweight large trees (e.g., via an MDL code length), encouraging posteriors supported on shallow/stable expressions.

\section{Symmetrization, Tail Bounds, and Talagrand-type Contractions}
\label{app:sym-tal}
This section collects standard ingredients used implicitly when translating Rademacher complexity into high-probability uniform bounds \cite{bartlett2002rademacher,mohri2018foundations,ledoux1991probability}.

\subsection{Rademacher symmetrization}
Let $\cF\subset(\cX\to\RR)$ and consider the empirical process supremum
\begin{equation}
  Z(S)\ :=\ \sup_{f\in\cF}\Big|\frac{1}{n}\sum_{i=1}^n \big(f(x_i)-\EE[f(X)]\big)\Big|.
\end{equation}
A symmetrization argument yields
\begin{equation}
  \EE[Z(S)] \ \le\ 2\,\EE_S\big[\Rad_S(\cF)\big],
\end{equation}
with $\Rad_S(\cF)$ as in the main text (expectation over independent Rademacher signs).
Thus controlling $\Rad_n(\cF)=\EE_S[\Rad_S(\cF)]$ controls the first moment of uniform deviations.

\subsection{McDiarmid bounded differences}
Assume $Z=Z(x_{1:n})$ satisfies bounded differences: changing one coordinate $x_i$ changes $Z$ by at most $c_i$.
Then for all $\varepsilon>0$,
\begin{equation}
  \PP\big(Z-\EE[Z]\ge \varepsilon\big)\ \le\ \exp\Big(-\frac{2\varepsilon^2}{\sum_{i=1}^n c_i^2}\Big).
\end{equation}
Applied with $c_i=\mathcal{O}(1/n)$ for bounded-range losses yields Gaussian concentration at rate $\mathcal{O}(n^{-1/2})$ after centering.

\subsection{Scalar contraction}
Let $\phi:\RR\to\RR$ be $L$-Lipschitz and $\phi(0)=0$.
Then for any finite sample $S$,
\begin{equation}
  \Rad_S(\phi\circ\cF)\ \le\ L\,\Rad_S(\cF).
\end{equation}
This is the one-dimensional specialization of \cref{eq:vector-contraction} used repeatedly when peeling unary operators off a composition chain.

\subsection{Unary composition chains}
\label{app:unary-chain}
Let $\cF_0\subset(\cX\to\RR)$ and define inductively $\cF_k=\{\phi_k\circ f:\ f\in\cF_{k-1}\}$ where each $\phi_k:\RR\to\RR$ is $L_k$-Lipschitz on the relevant range.
Then for every sample $S$,
\begin{equation}
  \Rad_S(\cF_k)\ \le\ \Big(\prod_{j=1}^k L_j\Big)\,\Rad_S(\cF_0).
\end{equation}
\begin{proof}
Induct on $k$, applying scalar contraction at each layer.
\end{proof}

\subsection{Dudley entropy integral (sketch)}
\label{app:dudley}
For many geometrically well-behaved classes, chaining yields
\begin{equation}
  \Rad_n(\cF)\ \lesssim\ \inf_{\alpha\ge0}\Big(\alpha + \frac{1}{\sqrt{n}}\int_{\alpha}^{\infty}\sqrt{\log \mathcal{N}(\varepsilon,\cF,\|\cdot\|_{n})}\,d\varepsilon\Big),
\end{equation}
where $\mathcal{N}(\varepsilon,\cF,\|\cdot\|_{n})$ is the $\varepsilon$-covering number of $\cF$ in the empirical $L_2(P_n)$ metric \cite{dudley1967sizes,vandervaart1996empirical,talagrand2005generic}.
Under Assumption~\ref{ass:affine-bound}, leaves lie in Lipschitz balls of radius $\mathcal O(W_{\max}X_{\max}\sqrt{p})$ over $\mathcal{X}$, while each remaining primitive has bounded subgraph complexity.
Thus $\cH_{\mathrm{comp}}^{1}$ is a finite union of uniformly Lipschitz families with bounded envelopes, giving $\Rad_n(\cH_{\mathrm{comp}}^{1})=\mathcal O(n^{-1/2})$ by standard VC/Rademacher chaining under fixed vocabulary size \cite{bartlett2002rademacher,mohri2018foundations}.

\subsection{Agnostic decomposition}
\label{app:agnostic}
Let $h^\star\in\arg\min_{h\in\cH_{\mathrm{comp}}^{d}} R(h)$ be a population risk minimizer within the class.
For any empirical risk minimizer $\wh h$,
\begin{equation}
  R(\wh h)-\inf_{h\in\cH} R(h)
  \ =\
  \big(R(\wh h)-\wh R_S(\wh h)\big)
  +\big(\wh R_S(\wh h)-\wh R_S(h^\star)\big)
  +\big(\wh R_S(h^\star)-R(h^\star)\big).
\end{equation}
The middle term is $\le 0$ by definition of ERM, while the two bracketed differences are controlled by uniform deviations over $\cH_{\mathrm{comp}}^{d}$, yielding the $\mathcal{O}(L^{d}/\sqrt{n})$ variance term under bounded loss/Lipschitzness after absorbing arity factors as in \cref{eq:rad-depth}.
If the truth lies outside $\cH_{\mathrm{comp}}^{d}$, add an irreducible approximation error $\inf_{h\in\cH_{\mathrm{comp}}^{d}}R(h)-R^\star$ where $R^\star$ is the Bayes risk under squared loss.

\section{Detailed Risk Bound for Bounded Squared Loss}
\label{app:risk-proof}
We sketch Theorem~\ref{thm:gen} with constants tracked only at the level of $\mathcal{O}(\cdot)$.
Assume $|y|\le B$ and $|f(x)|\le B$ uniformly over $\cH_{\mathrm{comp}}^{d}$.
For squared loss $\ell(\hat y,y)=(\hat y-y)^2$, on $[-B,B]^2$ one has $|\partial_{\hat y}\ell|\le 4B$, hence $\ell(\cdot,y)$ is $(4B)$-Lipschitz in $\hat y$.
Talagrand-type contraction for Lipschitz composition with squared loss implies $\Rad_S(\ell\circ\cH_{\mathrm{comp}}^{d})\le (4B)\cdot (Kb\sqrt{2}\,L)^{d-1}\Rad_S(\cH_{\mathrm{comp}}^{1})$ up to fixed vocabulary/arity factors (\cref{eq:rad-depth}).
Combining symmetrization, concentration on $\sup_{f\in\cH}(R(f)-\wh R_S(f))$, and bounded loss yields \cref{eq:main-bound}.

\section{Experimental Details and Reproducibility}
\label{app:exp}
The training code evaluates every combination of depths $d\in\{1,2,3,4\}$ and sample sizes $n\in\{50,100,500,1000,5000\}$ with multiple random seeds.
Models are trained with Adam on MSE until training loss stops improving (patience-based early stopping).
The Lipschitz proxy $\wh L$ is computed on a large batch of test-distribution inputs using either gradient norms or composed operator bounds.
The primary empirical summary plots the generalization gap against $(\wh L^d)/\sqrt{n}$.

\subsection{Wide-format numeric tables}
\label{app:wide-tables}
\Cref{tab:sweep-summary} reports mean$\pm$std over seeds for train/test error, gap magnitude, $\wh L$, and $(\wh L^{d})/\sqrt{n}$.
Wide numeric layouts are set in a smaller type size and scaled to the appendix text width so columns remain legible.
\Cref{tab:gap-grid} gives the same sweep as a compact median grid of $|\mathrm{gap}|$ over $(d,n)$, consistent with \cref{fig:heatmap-box}.

\input{tables/sweep_summary.tex}
\input{tables/gap_grid.tex}

\Cref{fig:factor-split} (factorization panels) and \cref{fig:heatmap-box} (heatmap and boxplots) appear in the main text.

\section{Appendix: Connection to SR Search and MDL}
Theorem~\ref{thm:gen} concerns statistical generalization for a fixed hypothesis class.
Symbolic regression additionally faces discrete search over structures, often guided by minimum-description-length (MDL) principles and sparsity \cite{rissanen1978mdl,brunton2016sindy}.
A promising direction is to combine our Lipschitz-depth certificate with MDL priors over trees to obtain bounds that trade off depth, stability, and description length.

\end{document}

%% file: tables/main_summary.tex
\begin{table*}[t]
\centering
\caption{Compact sweep summary (median over seeds).}
\label{tab:main-summary}
\begin{small}
\begin{tabular}{ccccc}
\toprule
Depth & $n$ & median $|\mathrm{gap}|$ & median $\hat L$ & median $(\hat L^d)/\sqrt{n}$ \\
\midrule
1 & 50 & 1.63e-05 & 1.30e+00 & 1.83e-01 \\
1 & 5000 & 1.25e-06 & 1.30e+00 & 1.83e-02 \\
\addlinespace
2 & 50 & 1.97e-05 & 1.77e+00 & 4.42e-01 \\
2 & 5000 & 1.45e-06 & 1.77e+00 & 4.43e-02 \\
\addlinespace
3 & 50 & 8.30e-05 & 7.40e+00 & 5.73e+01 \\
3 & 5000 & 5.99e-06 & 7.54e+00 & 6.07e+00 \\
\addlinespace
4 & 50 & 1.35e-04 & 1.01e+01 & 1.50e+03 \\
4 & 5000 & 5.53e-06 & 1.01e+01 & 1.46e+02 \\
\addlinespace
\bottomrule
\end{tabular}
\end{small}
\end{table*}

%% file: tables/sweep_summary.tex
\begin{table}[t]
\centering
\caption{Full sweep summary (mean$\pm$std over seeds).}
\label{tab:sweep-summary}
\begin{footnotesize}
\setlength{\tabcolsep}{4pt}
\resizebox{\linewidth}{!}{%
\begin{tabular}{@{}ccccccc@{}}
\toprule
$d$ & $n$ & Train MSE & Test MSE & $|\mathrm{gap}|$ & $\hat L$ & $(\hat L^{d})/\sqrt{n}$ \\
\midrule
1 & 50 & 9.44e-05$\pm$1.9e-05 & 1.04e-04$\pm$3.3e-06 & 1.73e-05$\pm$1.3e-05 & 1.28e+00$\pm$4.3e-01 & 1.81e-01$\pm$6.1e-02 \\
1 & 100 & 9.50e-05$\pm$1.3e-05 & 1.02e-04$\pm$2.0e-06 & 1.19e-05$\pm$8.3e-06 & 1.28e+00$\pm$4.3e-01 & 1.28e-01$\pm$4.3e-02 \\
1 & 500 & 1.02e-04$\pm$6.4e-06 & 1.01e-04$\pm$9.8e-07 & 5.16e-06$\pm$4.2e-06 & 1.28e+00$\pm$4.3e-01 & 5.72e-02$\pm$1.9e-02 \\
1 & 1000 & 1.01e-04$\pm$4.1e-06 & 1.01e-04$\pm$1.2e-06 & 3.75e-06$\pm$2.6e-06 & 1.28e+00$\pm$4.3e-01 & 4.05e-02$\pm$1.4e-02 \\
1 & 5000 & 1.00e-04$\pm$1.7e-06 & 1.00e-04$\pm$7.6e-07 & 1.55e-06$\pm$1.2e-06 & 1.28e+00$\pm$4.3e-01 & 1.81e-02$\pm$6.1e-03 \\
\addlinespace
2 & 50 & 9.08e-05$\pm$1.8e-05 & 1.08e-04$\pm$8.4e-06 & 2.08e-05$\pm$1.6e-05 & 1.75e+00$\pm$7.5e-01 & 5.08e-01$\pm$3.8e-01 \\
2 & 100 & 9.42e-05$\pm$1.2e-05 & 1.04e-04$\pm$3.8e-06 & 1.21e-05$\pm$6.9e-06 & 1.75e+00$\pm$7.5e-01 & 3.59e-01$\pm$2.7e-01 \\
2 & 500 & 1.01e-04$\pm$6.5e-06 & 1.01e-04$\pm$1.0e-06 & 5.04e-06$\pm$4.0e-06 & 1.75e+00$\pm$7.5e-01 & 1.60e-01$\pm$1.2e-01 \\
2 & 1000 & 1.01e-04$\pm$4.4e-06 & 1.01e-04$\pm$9.1e-07 & 3.47e-06$\pm$2.5e-06 & 1.75e+00$\pm$7.5e-01 & 1.14e-01$\pm$8.5e-02 \\
2 & 5000 & 1.00e-04$\pm$1.9e-06 & 1.01e-04$\pm$7.3e-07 & 1.63e-06$\pm$1.2e-06 & 1.75e+00$\pm$7.5e-01 & 5.08e-02$\pm$3.8e-02 \\
\addlinespace
3 & 50 & 6.24e-02$\pm$2.8e-01 & 8.31e-02$\pm$3.7e-01 & 2.07e-02$\pm$9.2e-02 & 8.12e+00$\pm$5.9e+00 & 2.09e+02$\pm$3.5e+02 \\
3 & 100 & 3.70e-04$\pm$4.6e-05 & 4.18e-04$\pm$1.3e-05 & 5.90e-05$\pm$2.4e-05 & 8.82e+00$\pm$5.8e+00 & 1.61e+02$\pm$2.4e+02 \\
3 & 500 & 4.04e-04$\pm$2.7e-05 & 4.05e-04$\pm$3.1e-06 & 2.17e-05$\pm$1.6e-05 & 8.84e+00$\pm$5.8e+00 & 7.27e+01$\pm$1.1e+02 \\
3 & 1000 & 4.03e-04$\pm$1.7e-05 & 4.03e-04$\pm$2.5e-06 & 1.38e-05$\pm$1.0e-05 & 8.84e+00$\pm$5.8e+00 & 5.15e+01$\pm$7.9e+01 \\
3 & 5000 & 4.01e-04$\pm$7.8e-06 & 4.02e-04$\pm$2.6e-06 & 6.65e-06$\pm$4.7e-06 & 8.84e+00$\pm$5.8e+00 & 2.30e+01$\pm$3.5e+01 \\
\addlinespace
4 & 50 & 1.37e-03$\pm$3.8e-03 & 1.52e-03$\pm$3.6e-03 & 2.26e-04$\pm$2.6e-04 & 1.18e+01$\pm$9.3e+00 & 2.26e+04$\pm$6.4e+04 \\
4 & 100 & 1.40e-03$\pm$3.8e-03 & 1.42e-03$\pm$3.4e-03 & 2.25e-04$\pm$4.5e-04 & 1.18e+01$\pm$9.3e+00 & 1.55e+04$\pm$4.3e+04 \\
4 & 500 & 6.83e-04$\pm$8.4e-04 & 7.03e-04$\pm$8.7e-04 & 3.67e-05$\pm$6.3e-05 & 1.20e+01$\pm$9.3e+00 & 7.05e+03$\pm$2.0e+04 \\
4 & 1000 & 5.01e-04$\pm$2.8e-04 & 5.04e-04$\pm$2.9e-04 & 1.63e-05$\pm$1.1e-05 & 1.21e+01$\pm$9.3e+00 & 4.99e+03$\pm$1.4e+04 \\
4 & 5000 & 4.19e-04$\pm$3.9e-05 & 4.20e-04$\pm$3.8e-05 & 6.44e-06$\pm$4.8e-06 & 1.21e+01$\pm$9.3e+00 & 2.25e+03$\pm$6.2e+03 \\
\addlinespace
\bottomrule
\end{tabular}%
}
\end{footnotesize}
\end{table}

%% file: tables/gap_grid.tex
\begin{table}[t]
\centering
\caption{Median $|\mathrm{gap}|$ (grid over seeds).}
\label{tab:gap-grid}
\begin{scriptsize}
\setlength{\tabcolsep}{3pt}
\resizebox{0.58\linewidth}{!}{%
\begin{tabular}{@{}rccccc@{}}
\toprule
\multicolumn{1}{r}{} & $50$ & $100$ & $500$ & $1000$ & $5000$ \\
\midrule
1 & 1.63e-05 & 9.67e-06 & 4.22e-06 & 3.05e-06 & 1.25e-06 \\
2 & 1.97e-05 & 1.18e-05 & 4.45e-06 & 2.74e-06 & 1.45e-06 \\
3 & 8.30e-05 & 6.09e-05 & 1.74e-05 & 1.03e-05 & 5.99e-06 \\
4 & 1.35e-04 & 7.58e-05 & 2.36e-05 & 1.19e-05 & 5.53e-06 \\
\bottomrule
\end{tabular}%
}
\end{scriptsize}
\end{table}

%% file: paper.bbl
\begin{thebibliography}{34}
\providecommand{\natexlab}[1]{#1}
\providecommand{\url}[1]{\texttt{#1}}
\expandafter\ifx\csname urlstyle\endcsname\relax
  \providecommand{\doi}[1]{doi: #1}\else
  \providecommand{\doi}{doi: \begingroup \urlstyle{rm}\Url}\fi

\bibitem[Alquier(2024)]{alquier2024pacbayes}
Alquier, P.
\newblock User-friendly introduction to {PAC}-{B}ayes bounds, 2024.
\newblock URL \url{https://arxiv.org/abs/2110.11216}.
\newblock Foundations and Trends in Machine Learning.

\bibitem[Bartlett \& Mendelson(2003)Bartlett and
  Mendelson]{bartlett2002rademacher}
Bartlett, P.~L. and Mendelson, S.
\newblock Rademacher and gaussian complexities: Risk bounds and structural
  results.
\newblock \emph{J. Mach. Learn. Res.}, 3:\penalty0 463--482, 2003.
\newblock URL \url{https://api.semanticscholar.org/CorpusID:463216}.

\bibitem[Bartlett et~al.(2017)Bartlett, Foster, and
  Telgarsky]{bartlett2017spectrally}
Bartlett, P.~L., Foster, D.~J., and Telgarsky, M.~J.
\newblock Spectrally-normalized margin bounds for neural networks, 2017.
\newblock URL \url{https://arxiv.org/abs/1706.08498}.

\bibitem[Brunton et~al.(2016)Brunton, Proctor, and Kutz]{brunton2016sindy}
Brunton, S.~L., Proctor, J.~L., and Kutz, J.~N.
\newblock Discovering governing equations from data by sparse identification of
  nonlinear dynamical systems.
\newblock \emph{Proceedings of the National Academy of Sciences}, 113\penalty0
  (15):\penalty0 3932--3937, 2016.
\newblock \doi{10.1073/pnas.1517384113}.

\bibitem[Catoni(2007)]{catoni2007pacbayes}
Catoni, O.
\newblock \emph{{PAC}-{B}ayesian Supervised Classification: The Thermodynamics
  of Statistical Learning}, volume~56 of \emph{Lecture Notes---Monograph
  Series}.
\newblock Institute of Mathematical Statistics, Beachwood, Ohio, 2007.
\newblock ISBN 978-0-940600-68-7.
\newblock \doi{10.1214/074921707000000391}.
\newblock URL \url{https://doi.org/10.1214/074921707000000391}.

\bibitem[Cranmer(2023)]{cranmer2023pysr}
Cranmer, M.
\newblock Interpretable machine learning for science with {PySR} and
  {S}ymbolic{R}egression.jl, 2023.
\newblock URL \url{https://arxiv.org/abs/2305.01582}.

\bibitem[Cranmer et~al.(2020)Cranmer, Sanchez-Gonzalez, Battaglia, Xu, Cranmer,
  Spergel, and Ho]{cranmer2020discovering}
Cranmer, M., Sanchez-Gonzalez, A., Battaglia, P., Xu, R., Cranmer, K., Spergel,
  D., and Ho, S.
\newblock Discovering symbolic models from deep learning with inductive biases,
  2020.
\newblock URL \url{https://arxiv.org/abs/2006.11287}.

\bibitem[Dudley(1967)]{dudley1967sizes}
Dudley, R.~M.
\newblock The sizes of compact subsets of hilbert space and continuity of
  gaussian processes.
\newblock \emph{Journal of Functional Analysis}, 1:\penalty0 125--165, 1967.
\newblock URL \url{https://api.semanticscholar.org/CorpusID:122249056}.

\bibitem[Kanade et~al.(2024)Kanade, Rebeschini, and
  Vaskevicius]{kanade2024exponential}
Kanade, V., Rebeschini, P., and Vaskevicius, T.
\newblock Exponential tail local {R}ademacher complexity risk bounds without
  the {B}ernstein condition.
\newblock \emph{Journal of Machine Learning Research}, 25\penalty0
  (388):\penalty0 1--43, 2024.
\newblock URL \url{http://jmlr.org/papers/v25/23-0063.html}.

\bibitem[Koltchinskii \& Panchenko(2002)Koltchinskii and
  Panchenko]{koltchinskii2002local}
Koltchinskii, V. and Panchenko, D.
\newblock Empirical margin distributions and bounding the generalization error
  of combined classifiers.
\newblock \emph{Annals of Statistics}, 30:\penalty0 1--50, 2002.
\newblock URL \url{https://api.semanticscholar.org/CorpusID:2307733}.

\bibitem[Koza(1992)]{koza1992genetic}
Koza, J.~R.
\newblock \emph{Genetic Programming: On the Programming of Computers by Means
  of Natural Selection}.
\newblock MIT Press, Cambridge, MA, 1992.
\newblock ISBN 978-0-262-11170-6.

\bibitem[La~Cava et~al.(2021)La~Cava, Orzechowski, Burlacu, Olivetti~de
  Fran{\c{c}}a, Virgolin, Jin, Kommenda, and Moore]{la2021srbench}
La~Cava, W., Orzechowski, P., Burlacu, B., Olivetti~de Fran{\c{c}}a, F.,
  Virgolin, M., Jin, Y., Kommenda, M., and Moore, J.~H.
\newblock Contemporary symbolic regression methods and their relative
  performance, 2021.
\newblock URL \url{https://arxiv.org/abs/2107.14351}.
\newblock NeurIPS 2021 Datasets and Benchmarks Track; introduces the {SRBench}
  suite.

\bibitem[Langford \& Seeger(2001)Langford and Seeger]{langford2001pacbayes}
Langford, J. and Seeger, M.
\newblock Bounds for averaging classifiers.
\newblock 02 2001.

\bibitem[Ledoux \& Talagrand(1991)Ledoux and Talagrand]{ledoux1991probability}
Ledoux, M. and Talagrand, M.
\newblock \emph{Probability in {B}anach Spaces: Isoperimetry and Processes}.
\newblock Springer-Verlag, Berlin, 1991.
\newblock ISBN 978-3-540-52013-9.
\newblock \doi{10.1007/978-3-642-20212-4}.

\bibitem[Lei et~al.(2016)Lei, Ding, and Bi]{lei2016local}
Lei, Y., Ding, L., and Bi, Y.
\newblock Local {R}ademacher complexity bounds based on covering numbers.
\newblock \emph{Neurocomputing}, 218:\penalty0 320--330, 2016.
\newblock \doi{10.1016/j.neucom.2016.08.074}.

\bibitem[Makke \& Chawla(2024)Makke and Chawla]{makke2024symbolic}
Makke, N. and Chawla, S.
\newblock Interpretable scientific discovery with symbolic regression: a
  review.
\newblock \emph{Artificial Intelligence Review}, 57\penalty0 (1), 2024.
\newblock \doi{10.1007/s10462-023-10622-0}.
\newblock URL \url{https://doi.org/10.1007/s10462-023-10622-0}.
\newblock Article number 2.

\bibitem[Maurer(2016{\natexlab{a}})]{maurer2016chain}
Maurer, A.
\newblock A chain rule for the expected suprema of {G}aussian processes.
\newblock \emph{Theoretical Computer Science}, 650:\penalty0 109--122,
  2016{\natexlab{a}}.
\newblock \doi{10.1016/j.tcs.2016.07.034}.

\bibitem[Maurer(2016{\natexlab{b}})]{maurer2016vectorcontraction}
Maurer, A.
\newblock A vector-contraction inequality for {R}ademacher complexities,
  2016{\natexlab{b}}.
\newblock URL \url{https://arxiv.org/abs/1605.00251}.

\bibitem[Maurer \& Pontil(2011)Maurer and Pontil]{maurer2011structured}
Maurer, A. and Pontil, M.
\newblock Structured sparsity and generalization, 2011.
\newblock URL \url{https://arxiv.org/abs/1108.3476}.

\bibitem[McAllester(1999)]{mcallester1999pacbayes}
McAllester, D.~A.
\newblock Some {PAC}-{B}ayesian theorems.
\newblock \emph{Machine Learning}, 37\penalty0 (3):\penalty0 355--363, 1999.
\newblock \doi{10.1023/A:1007618624809}.
\newblock URL \url{https://doi.org/10.1023/A:1007618624809}.

\bibitem[Mohri et~al.(2018)Mohri, Rostamizadeh, and
  Talwalkar]{mohri2018foundations}
Mohri, M., Rostamizadeh, A., and Talwalkar, A.
\newblock \emph{Foundations of Machine Learning}.
\newblock MIT Press, Cambridge, MA, 2nd edition, 2018.
\newblock ISBN 9780262039406.

\bibitem[Neyshabur et~al.(2017)Neyshabur, Bhojanapalli, McAllester, and
  Srebro]{neyshabur2017exploring}
Neyshabur, B., Bhojanapalli, S., McAllester, D., and Srebro, N.
\newblock Exploring generalization in deep learning, 2017.
\newblock URL \url{https://arxiv.org/abs/1706.08947}.

\bibitem[Petersen et~al.(2021)Petersen, Landajuela, Mundhenk, Santiago, Kim,
  and Kim]{petersen2021equation}
Petersen, B.~K., Landajuela, M., Mundhenk, T.~N., Santiago, C.~P., Kim, S.~K.,
  and Kim, J.~T.
\newblock Deep symbolic regression: Recovering mathematical expressions from
  data via risk-seeking policy gradients, 2021.
\newblock URL \url{https://arxiv.org/abs/1912.04871}.
\newblock Published at ICLR 2021.

\bibitem[Rissanen(1978)]{rissanen1978mdl}
Rissanen, J.
\newblock Modeling by shortest data description.
\newblock \emph{Automatica}, 14\penalty0 (5):\penalty0 465--471, 1978.
\newblock \doi{10.1016/0005-1098(78)90005-5}.

\bibitem[Schmidt \& Lipson(2009{\natexlab{a}})Schmidt and
  Lipson]{schmidt2009distilling}
Schmidt, M.~D. and Lipson, H.
\newblock Distilling free-form natural laws from experimental data.
\newblock \emph{Science}, 324\penalty0 (5923):\penalty0 81--85,
  2009{\natexlab{a}}.
\newblock \doi{10.1126/science.1165893}.

\bibitem[Schmidt \& Lipson(2009{\natexlab{b}})Schmidt and
  Lipson]{schmidt2009implicit}
Schmidt, M.~D. and Lipson, H.
\newblock Symbolic regression of implicit equations.
\newblock In \emph{Genetic Programming Theory and Practice VII}, Genetic and
  Evolutionary Computation, pp.\  73--85. Springer, New York,
  2009{\natexlab{b}}.
\newblock \doi{10.1007/978-1-4419-1626-6\_5}.
\newblock URL \url{https://doi.org/10.1007/978-1-4419-1626-6\_5}.

\bibitem[Seeger(2002)]{seeger2002pacbayes}
Seeger, M.
\newblock Pac-bayesian generalization error bounds for gaussian process
  classification.
\newblock \emph{Journal of Machine Learning Research}, 3, 08 2002.
\newblock \doi{10.1162/153244303765208386}.

\bibitem[Shalev-Shwartz \& Ben-David(2014)Shalev-Shwartz and
  Ben-David]{shalev2014understanding}
Shalev-Shwartz, S. and Ben-David, S.
\newblock \emph{Understanding Machine Learning: From Theory to Algorithms}.
\newblock Cambridge University Press, 2014.
\newblock ISBN 978-1-107-05713-5.
\newblock \doi{10.1017/CBO9781107298019}.

\bibitem[Talagrand(2014)]{talagrand2005generic}
Talagrand, M.
\newblock \emph{Upper and Lower Bounds for Stochastic Processes: Modern Methods
  and Classical Problems}.
\newblock Springer, Heidelberg, 2 edition, 2014.
\newblock ISBN 978-3-642-54074-5.
\newblock \doi{10.1007/978-3-642-54075-2}.

\bibitem[Truong(2022)]{truong2022rademacher}
Truong, L.~V.
\newblock On {R}ademacher complexity-based generalization bounds for deep
  learning, 2022.
\newblock URL \url{https://arxiv.org/abs/2208.04284}.

\bibitem[Udrescu \& Tegmark(2020)Udrescu and Tegmark]{udrescu2020aifeynman}
Udrescu, S.-M. and Tegmark, M.
\newblock {AI} {F}eynman: A physics-inspired method for symbolic regression,
  2020.
\newblock URL \url{https://arxiv.org/abs/1905.11481}.

\bibitem[Valiant(1984)]{valiant1984theory}
Valiant, L.~G.
\newblock A theory of the learnable.
\newblock \emph{Commun. ACM}, 27:\penalty0 1134--1142, 1984.
\newblock URL \url{https://api.semanticscholar.org/CorpusID:59712}.

\bibitem[van~der Vaart \& Wellner(1996)van~der Vaart and
  Wellner]{vandervaart1996empirical}
van~der Vaart, A.~W. and Wellner, J.~A.
\newblock \emph{Weak Convergence and Empirical Processes: With Applications to
  Statistics}.
\newblock Springer Series in Statistics. Springer, 1996.
\newblock ISBN 978-0-387-94640-5.
\newblock \doi{10.1007/978-1-4757-2545-2}.

\bibitem[Vapnik(1998)]{vapnik1998statistical}
Vapnik, V.~N.
\newblock \emph{Statistical Learning Theory}.
\newblock Wiley, New York, NY, USA, September 1998.

\end{thebibliography}
